\documentclass[10pt]{article}

\usepackage[hmargin=35mm,vmargin=40mm]{geometry}
\usepackage{amssymb}
\usepackage{amsmath}
\usepackage{algorithm}
\usepackage{amsthm}
\usepackage{graphicx}
\usepackage{subfigure}
\usepackage[noend]{algpseudocode} 

\providecommand{\U}[1]{\protect\rule{.1in}{.1in}}
\newtheorem{theorem}{Theorem}
\begin{document}

\title{Online Learning for Classification of Low-rank Representation Features and Its Applications in Audio Segment Classification}
\author{Ziqiang Shi,~Jiqing Han,~Tieran Zheng,~Shiwen Deng}
\maketitle

\begin{abstract}
In this paper, a novel framework based on trace norm minimization for audio segment is proposed. In this framework, both the feature extraction and classification are obtained by solving corresponding convex optimization problem with trace norm regularization. For feature extraction, robust principle component analysis (robust PCA) via minimization a combination of the nuclear norm and the $\ell_1$-norm is used to extract low-rank features which are robust to white noise and gross corruption for audio segments. These low-rank features are fed to a linear classifier where the weight and bias are learned by solving similar trace norm constrained problems. For this classifier, most methods find the weight and bias in batch-mode learning, which makes them inefficient for large-scale problems.
In this paper, we propose an online framework using
accelerated proximal gradient method. This framework has a main advantage in memory cost. In addition, as a result of the regularization formulation of matrix classification, the Lipschitz constant was given explicitly, and hence the step size estimation of
general proximal gradient method was omitted in our approach. Experiments on real data sets for laugh/non-laugh and applause/non-applause classification indicate that this novel framework is effective and noise robust.
\end{abstract}


\section{Introduction}


Audio feature extraction and classification methods
have been studied by many researchers over the
years~\cite{Lu2002,Cui2003,Pradeep2006,Umapathy2007}. In general, audio classification can be performed in two steps, which involves reducing the audio sound to a small set of parameters using various feature extraction techniques and classifying or categorizing over these parameters. Feature commonly exploited for audio classification can be roughly classified into time domain features, transformation domain features, time-transformation domain features or their combinations~\cite{Umapathy2007, Zhuang2008}. Many of those features are common to audio signal processing and speech recognition and have many successful performances in various applications. However almost all these features are based on short time duration and in vector form (it is easy to handle but sometimes not proper), although it is believed that long time duration (seconds) help a lot in decision making. In this work we will build robust features on a long time duration in matrix form which is the most natural way using long time audio information.

In order to map or smooth the audio segment into a robust matrix space, we introduce the trace norm regularization technique to audio signal processing. The trace norm regularization is a principled approach to learn low-rank matrices through convex optimization problems~\cite{Fazel2001}. These similar problems arise in many machine learning tasks such as matrix completion~\cite{Srebro2005}, multi-task learning~\cite{Argyriou2008}, robust principle component antilysis (robust PCA)~\cite{Wright2009,Lin2010}, and matrix classification~\cite{Tomioka2007}. In this paper, robust PCA is used to extract matrix representation features for audio segments. Unlike traditional frame based vector features, these matrix features are extracted based on sequences of audio frames. It is believed that in a short duration the signals are contributed by a few factors. Thus it is natural to approximate the frame sequence by low-rank features using robust PCA which assumes that the observed matrices are combinations of some low-rank matrices and some corruption noise matrices.

Having extracted descriptive features, various machine learning methods are used to provide a final classification of the audio events such as rule-based approaches, Gaussian mixture models, support vector machines, Bayesian networks, and etc.~\cite{Umapathy2007, Zhuang2008, Guo2003}. In most previous work, these two steps for audio classification are always separate and independent. In this work, we can learn the classifiers in solving similar optimization problems using trace norm regularization. After extraction of the robust low-rank matrix feature, the regularization framework based matrix classification approach proposed by Tomioka and Aihara in~\cite{Tomioka2007} is used to predict the label.


The problem of matrix classification (MC) with spectral regularization was first
proposed by Tomioka and Aihara in~\cite{Tomioka2007}. The goal of the problem is to infer the weight matrix and bias under low trace norm constraints and low deviation
of the empirical statistics from their predictions. The trace norm was use to
measure the complexity of the weight matrix of the linear classifier for matrix classifications. This kind of inference task belongs to the more general
problem of learning low-rank matrix through convex optimization. For the matrix rank minimization is NP-hard in
general due to the combinatorial nature of the rank function, a commonly-used convex relaxation of the rank function is the trace norm (nuclear norm)~\cite{Fazel2001},
defined as the sum of the singular values of the matrix.

Recent related researches are not focused on matrix classification directly, but rather on general trace norm minimization problem~\cite{Toh2010,Ji2009,Liu2009}. These general algorithm can be adapted to matrix classification suitably. In these methods, most are iterative \emph{batch} procedures~\cite{Toh2010,Ji2009,Liu2009}, accessing the whole training set at each iteration in order to minimize a weighted sum of a cost function and the trace norm. This kind of learning procedure cannot deal with huge size training set for the data probably cannot be loaded into memory simultaneously. Furthermore it cannot be started until the training data are prepared, hence cannot effectively deal with training data appear in sequence, such as audio and video processing.

To address these problems, we propose an \emph{online} approach that processes the training samples, one at a time, or in mini-batches to learn the weight matrix and the bias for matrix classification. We transform the general batch-mode accelerated proximal gradient (APG)~\cite{Toh2010,Ji2009} method for trace norm minimization to the online learning framework. In this online learning framework, a slight improvement over the exact APG leads an inexact APG (IAPG) method, which needs less computation in one iteration than using exact APG. In addition, as a special case of general convex optimization problem, we derived the closed-form of the Lipschitz constant, hence the step size estimation~\cite{Toh2010,Ji2009} of the general APG method was omitted in our approach.

Our main contributions in this work can be summarized as follows:

\begin{enumerate}
  \item To our best knowledge, we are the first to introduce low-rank constraints in audio and speech signal processing, and the results show that these constrains make the systems more robust to noise, especially to large corruptions.
  \item We propose online learning algorithms to learn the trace norm minimization based matrix classifier, which make the approaches work in real applications.
\end{enumerate}

The paper is organized as follows: Section~\ref{sec:MatrixFeature} presents the extraction of matrix representation feature. Section~\ref{sec:MatrixClassifyAED} presents the matrix classification problem solving via the general APG method and the proposed audio event detection with matrix classification. The proposed online methods with exact and inexact APG for weight and bias learning are introduced in Section~\ref{sec:OnlineLearning}. Section~\ref{sec:Experimental} is devoted to experimental results to demonstrate the characteristics and merits of the proposed algorithm. Finally we give some concluding remarks in Section~\ref{sec:Conclusions}.

\section{Low-Rank Matrix Representation Features}
\label{sec:MatrixFeature}
Over the past decades, a lot work has been done on audio and speech features for audio and speech processing~\cite{Cui2003,Pradeep2006,Zhuang2008}. Due to convenience and the short-time stationary assumption, these features are mainly in vector form based on frames, although it is believed that features based on longer duration help a lot in decision making. In order to build long term features, the consecutive frame signals are made together as rows, then the audio segments become matrices. Generally, it is assumed and believed that the consecutive frame signals are influenced by a few factors, thus these matrices are combinations of low-rank components and noise. Hence it is natural to approximate these matrices by low-rank matrices. In this work, transformations of these approximate low-rank matrices are used as features.

Given an observed data matrix $D\in\mathbb{R}^{m\times n}$, where $m$ is the number of frames and $n$ represents the number of samples in a frame, it is assumed that it can be decomposed as
\begin{equation}\label{eq:DecompAE}
D=A+E,
\end{equation}
where $A$ is the low-rank component and $E$ is the error or noise matrix. The purpose here is to recover the low-rank component without knowing the rank of it. For this problem, PCA is a suitable approach that it can find the low-dimensional approximating subspace by forming a low-rank approximation to the data matrix~\cite{Jolliffe1986}. However, it breaks down under large corruption, even if that corruption affects only a very few of the observation which is often encountered in practice~\cite{Lin2010}. To solve this problem, the following convex optimization formulation is proposed
 \begin{equation}\label{eq:RobustPCAFormulation}
\min_{A,E\in\mathbb{R}^{m\times n}} \|A\|_*+\lambda\|E\|_1, \textrm{ subject to } D=A+E,
\end{equation}
where $\|\cdot\|_*$ denotes the trace norm of a matrix which is defined as the sum of the singular values, $\|\cdot\|_1$ denotes the sum of the absolute values of matrix elements, and $\lambda$ is a positive regularization parameter. This optimization is refereed to as \emph{robust PCA} in~\cite{Wright2009} for its ability to exactly recover underlying low-rank structure in data even in the presence of large errors or outliers. In order to solve Equation~(\ref{eq:RobustPCAFormulation}), several algorithms have been proposed, among which the augmented Lagrange multiplier method is the most efficient and accurate at present~\cite{Lin2010}. In our work, this robust PCA method is employed for the low-rank matrix extraction.

In order to apply the augmented Lagrange multiplier (ALM) to the robust PCA problem, Lin et. al.~\cite{Lin2010} identify the problem as
 \begin{equation}\label{eq:ALMforRPCA}
X=(A,E), f(X)=\|A\|_*+\lambda\|E\|_1, \textrm{ and } h(X)=D-A-E,
\end{equation}
and the Lagrangian function becomes
\begin{equation}\label{eq:LagFuncforALM}
L(A,E,Y,\mu)\doteq \|A\|_*+\lambda\|E\|_1 
+<Y,D-A-E>+\frac{\mu}{2}\|D-A-E\|_F^2.
\end{equation}
Two ALM algorithms to solve the above formulation are proposed in~\cite{Lin2010}. Considering a balance between processing speed and accuracy, the robust PCA via the inexact ALM method is chosen in our work. Thus the matrix representation feature extraction process based on this approach is summarized in Algorithm~\ref{algo:PRCAviaIALM}. In Algorithm~\ref{algo:PRCAviaIALM}, $J(D)$ is defined as the larger one of $\|D\|_2$ and $\lambda^{-1}\|D\|_{\infty}$, where $\|\cdot\|_{\infty}$ is the maximum absolute value of the matrix elements. The $\mathcal{S}_{\varepsilon}[\cdot]$ is the soft-thresholding operator introduced in~\cite{Lin2010}.

Fig.~\ref{fig:RPCAtoAudioSeg} shows the recovered low-rank matrices via applying robust PCA to the matrix form of a typical laugh sound effect audio segment with or without corruptions. In which, the regularization parameter is fixed as 1. It can be seen that robust PCA extracted matrices are robust to large errors and Gaussian noise. Ideally, these above recovered low-rank matrices can be used as features directly. But in order to balance the speed and performance, in this work the we transform the recovered low-rank matrices into MFCCs (mel-frequency cepstral coefficients) matrices. All rows in the low-rank matrices are transformed into MFCCs independently. Fig.~\ref{fig:Spectrograms_RPCAtoAudioSeg} shows the spectrograms of the signal in Fig.~\ref{fig:RPCAtoAudioSeg} respectively. It seems that the spectrograms of the low-rank components vary not much compare to the spectrograms of the corrupted signals.

\begin{algorithm}
\label{algo:PRCAviaIALM}
Recovering of Low-rank Component from Audio Segments via RPCA.

\textbf{Input}: $D\in \mathbb{R}^{m\times n}$ (matrix form of the audio segment).

\textbf{Initialize}: $D\in \mathbb{R}^{m\times n}, Y_0=D/J(D), E_0=0, \mu_0>0, \rho>1, k=0.$

1: \textbf{while} not converged \textbf{do}

2: // Lines 3-4 solve $A_{k+1}=\mbox{arg}\mathop {\min}\limits_{A} L(A,E_k,Y_k,\mu_k).$

3: $(U,S,V)=\textrm{svd}(D-E_k+\mu_k^{-1}Y_k)$.

4: $A_k=U\mathcal{S}_{\mu_k^{-1}}[S]V^T$.

5: // Line 6 solves $E_{k+1}=\mbox{arg}\mathop {\min}\limits_{E} L(A_{k+1},E,Y_k,\mu_k).$

6: $E_{k+1}=\mathcal{S}_{\lambda\mu_k^{-1}}[D-A_{k+1}+\mu_k^{-1}Y_k].$

7: $Y_{k+1}=Y_k+\mu_k(D-A_{k+1}-E_{k+1})$.

8: Update $\mu_k$ to $\mu_{k+1}$.

9: $k\leftarrow k+1$.

10: \textbf{end while}

\textbf{Output}: $W\leftarrow W_k$.
\end{algorithm}

\begin{figure}

\subfigure[]{
\includegraphics[width=0.4\textwidth]{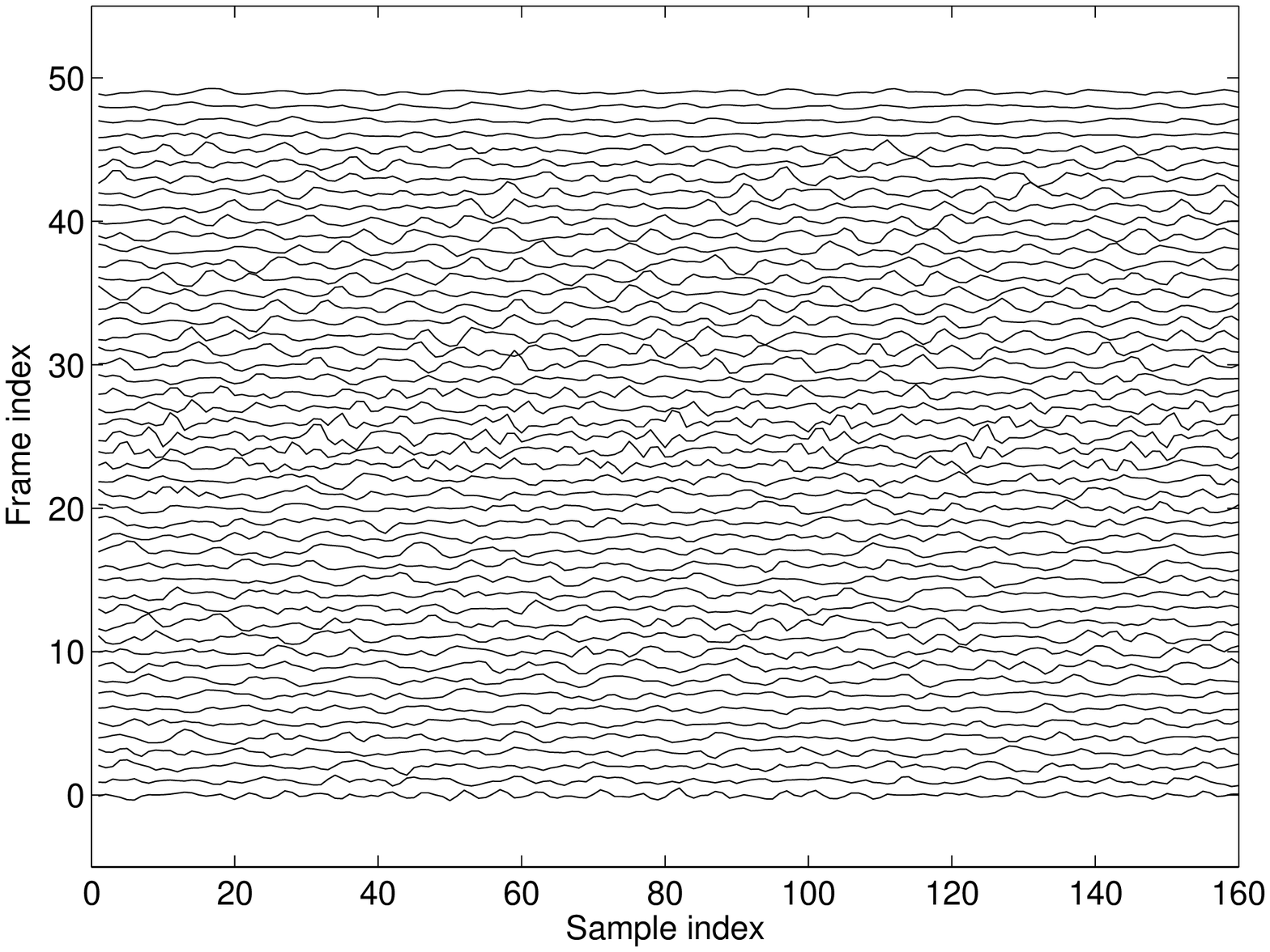}
\label{fig:MatLau}
}
\hspace{0.1in}
\subfigure[]{
\includegraphics[width=0.4\textwidth]{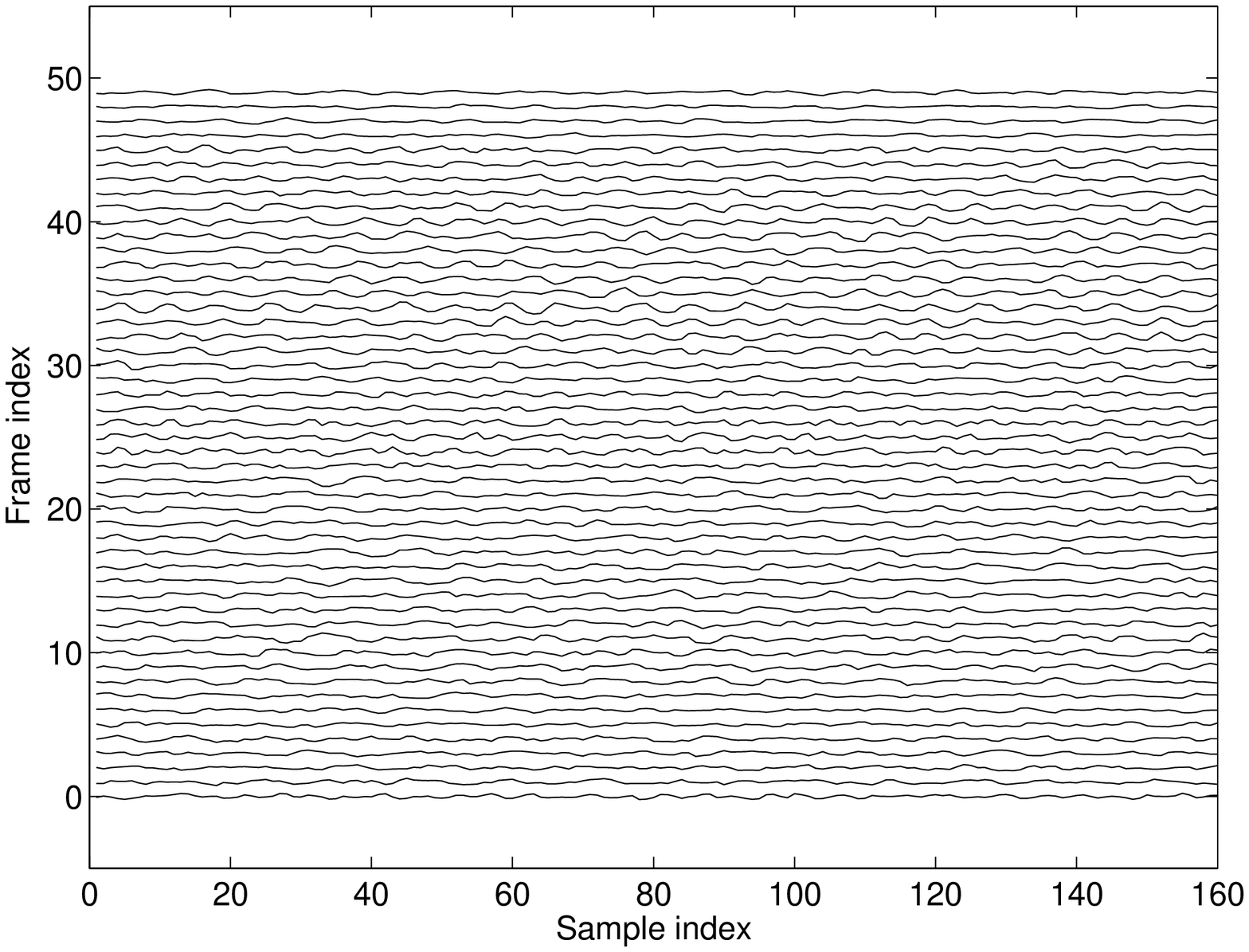}
}
\hspace{0.1in}
\subfigure[]{
\includegraphics[width=0.4\textwidth]{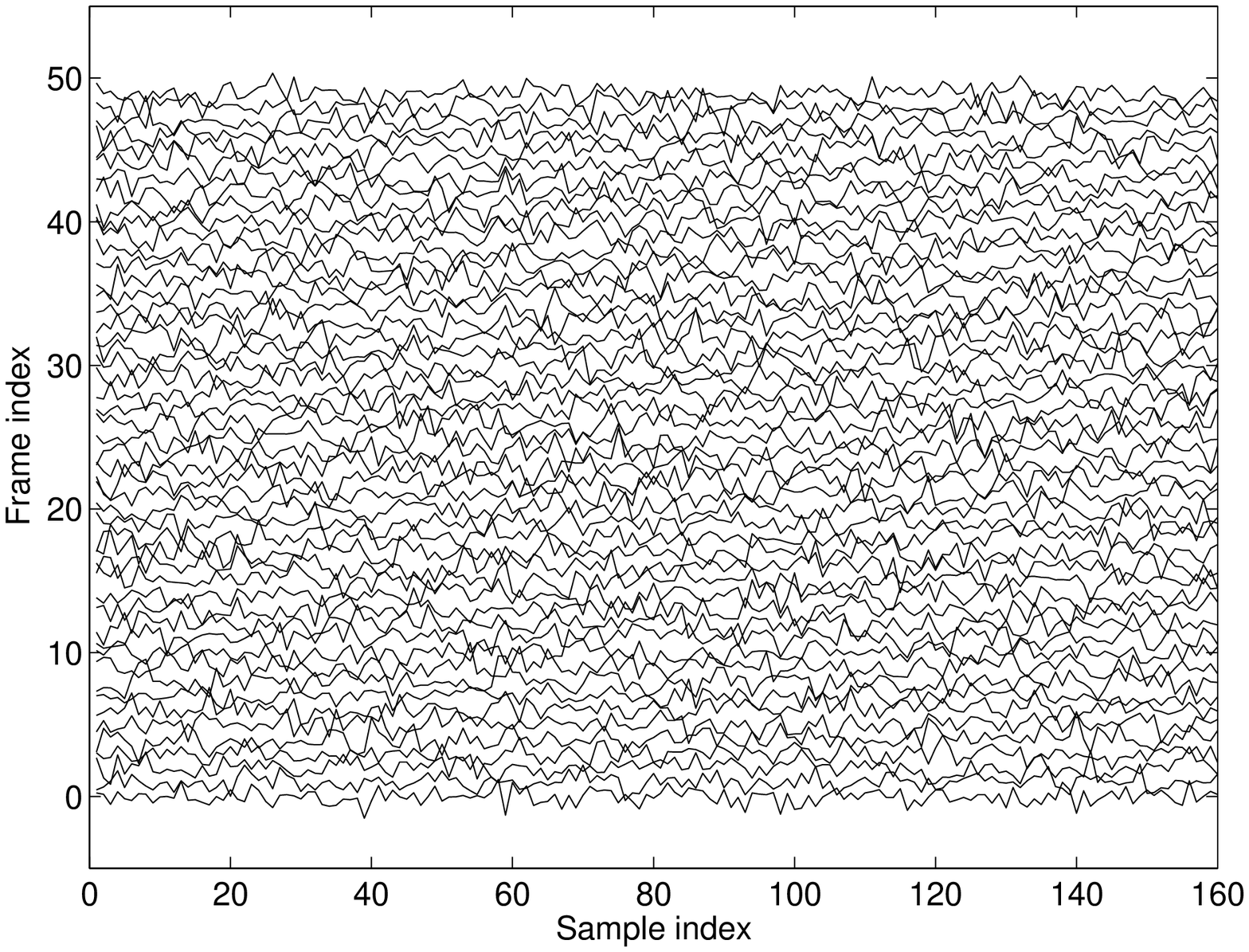}
}
\hspace{0.1in}
\subfigure[]{
\includegraphics[width=0.4\textwidth]{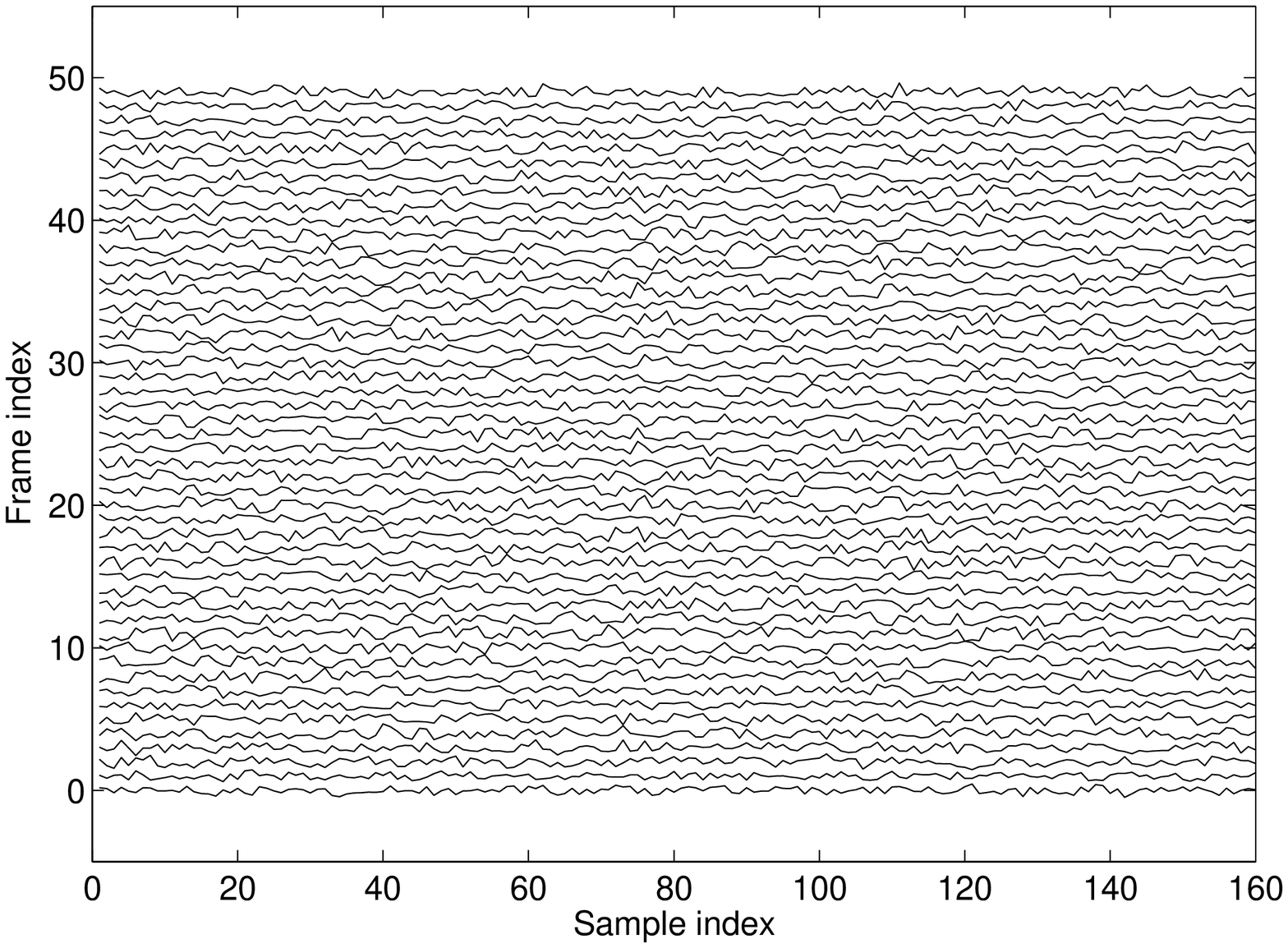}
}
\hspace{0.1in}
\subfigure[]{
\includegraphics[width=0.4\textwidth]{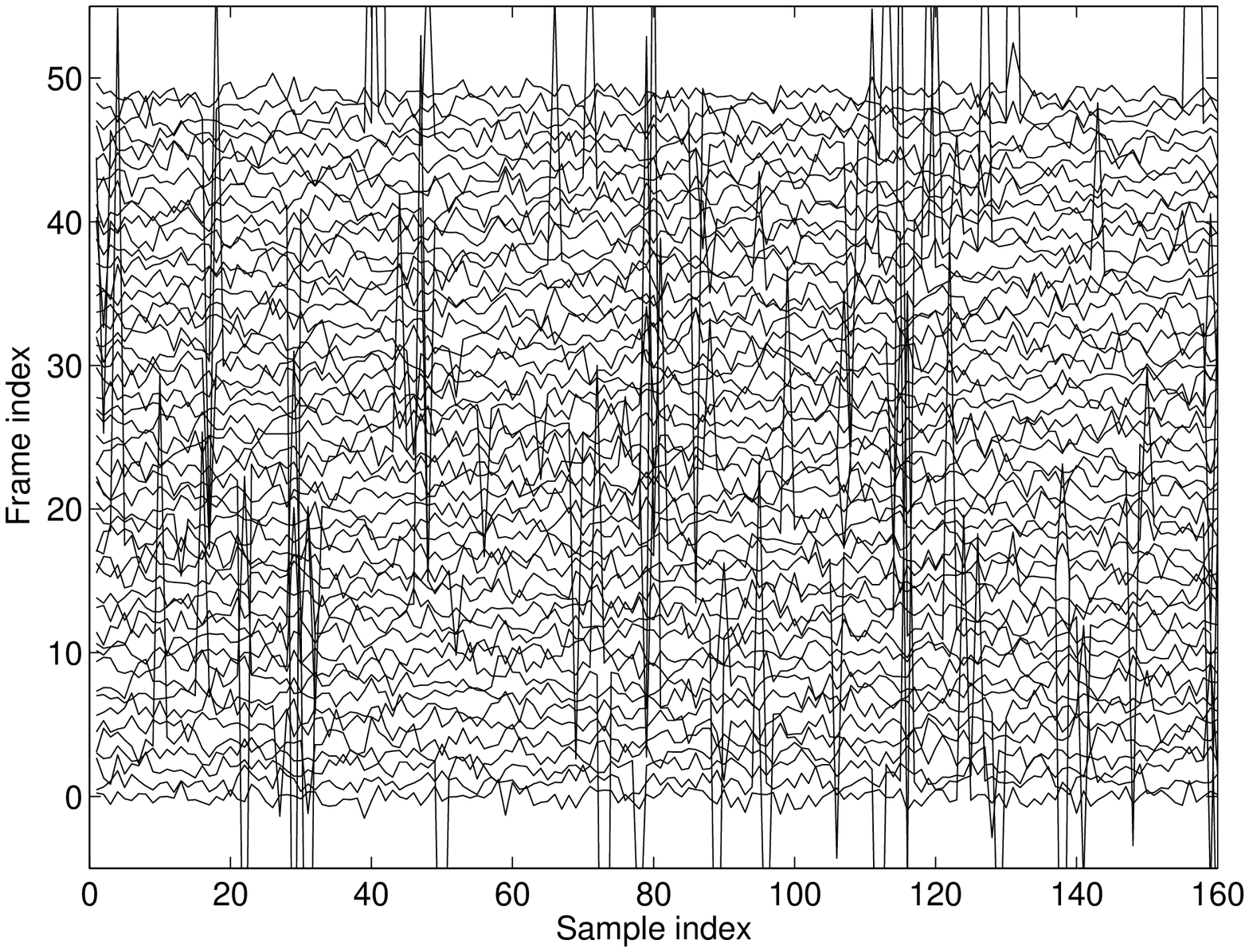}
}
\hspace{1in}
\subfigure[]{
\includegraphics[width=0.4\textwidth]{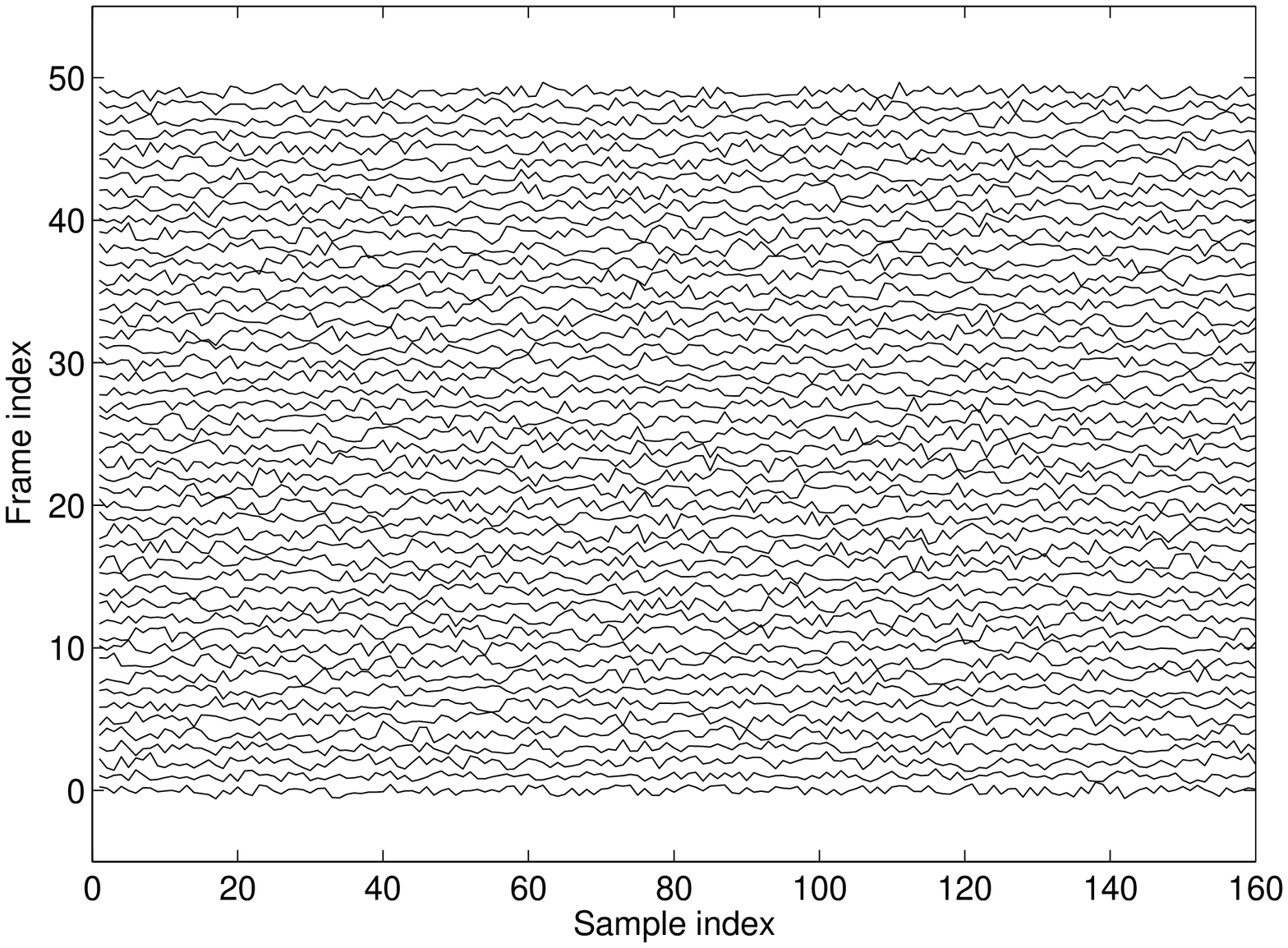}
}
\caption{Matrix form of audio segments with or without noise and extracted matrix features via Robust PCA with $\lambda = 1$ throughout. (a) Matrix form of a typical laugh sound effect audio segment; (b) The low-rank component recovered from (a) via robust PCA; (c) Matrix form of the same audio segment corrupted by white Gaussian noise with SNR=20dB; (d) The low-rank component recovered from (c) via robust PCA; (e) Matrix form of the same audio segment corrupted by white Gaussian noise and random large errors; (f) The low-rank component recovered from (e) via robust PCA.}
\label{fig:RPCAtoAudioSeg}       
\end{figure}

\begin{figure}

\subfigure[]{
\includegraphics[width=0.4\textwidth]{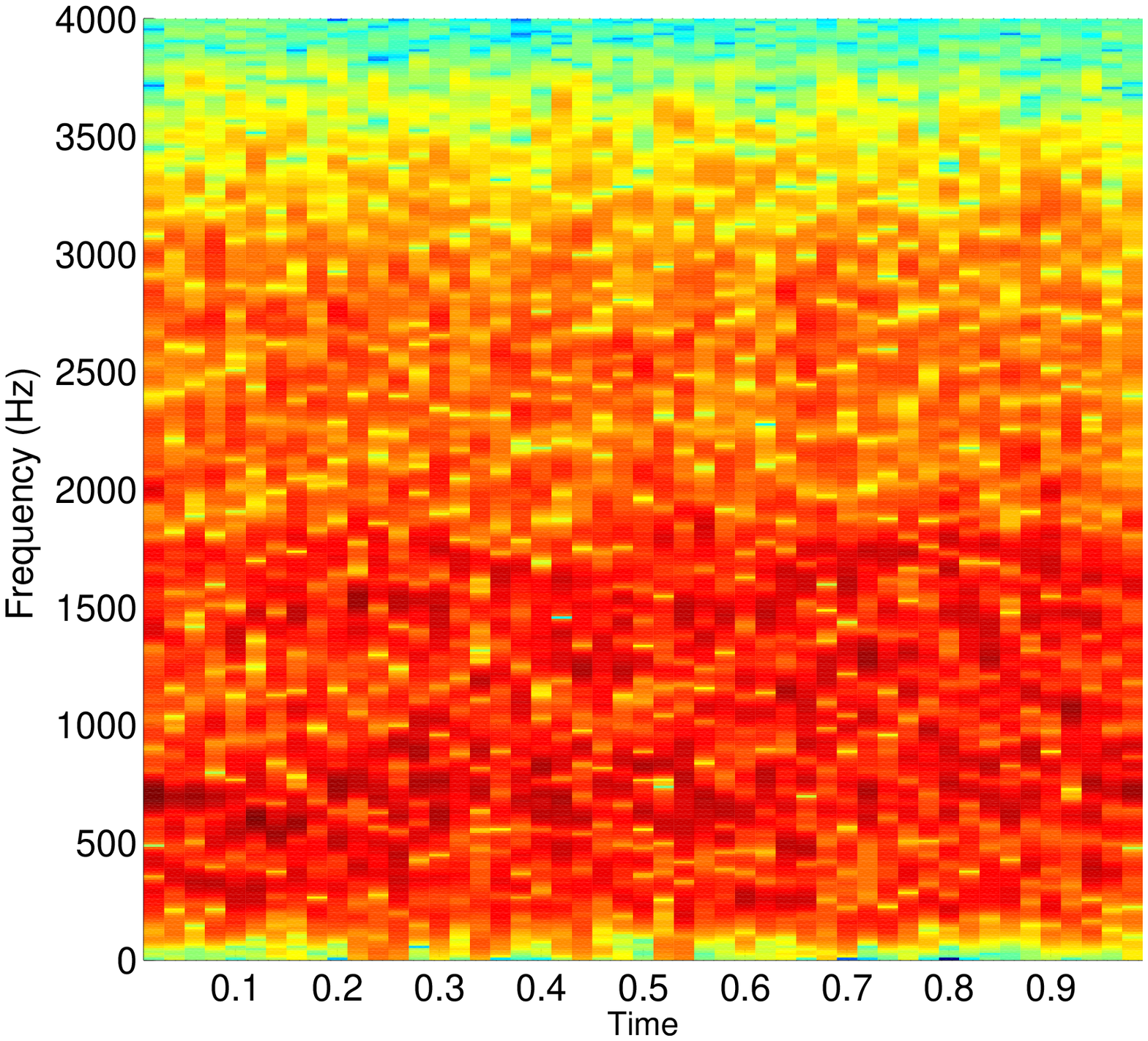}
\label{fig:MatLau}
}
\subfigure[]{
\includegraphics[width=0.4\textwidth]{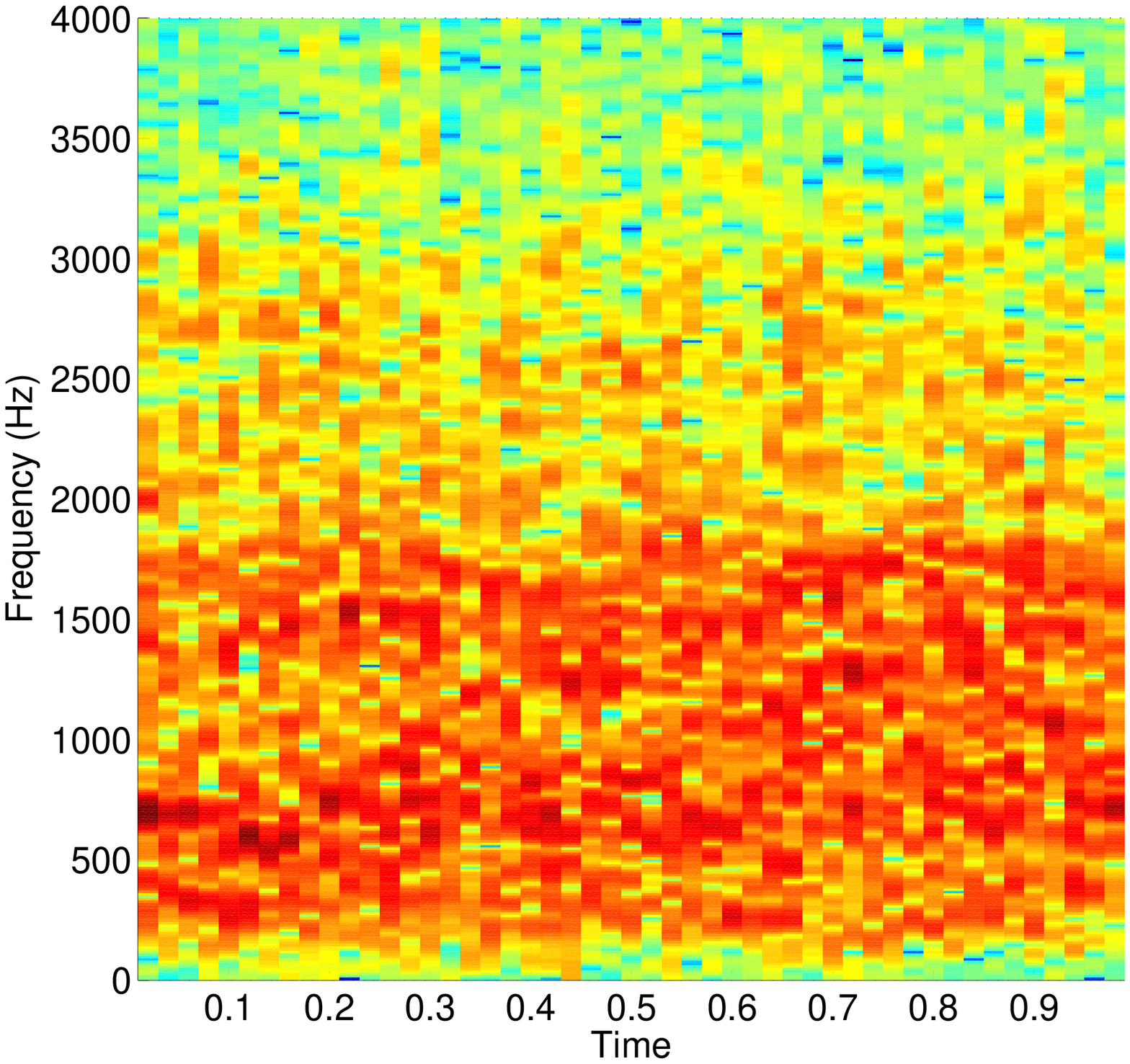}
}
\subfigure[]{
\includegraphics[width=0.4\textwidth]{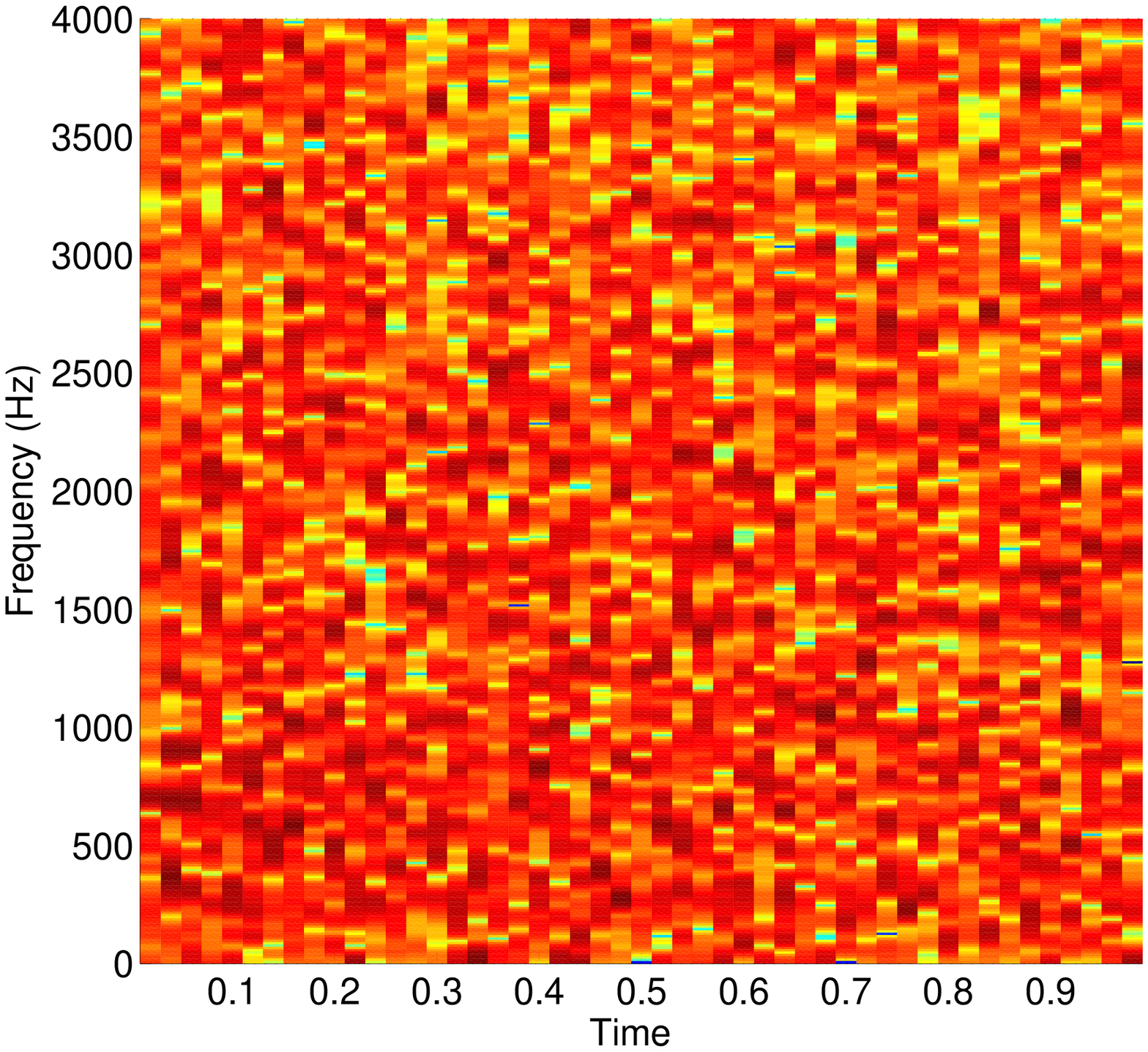}
}
\subfigure[]{
\includegraphics[width=0.4\textwidth]{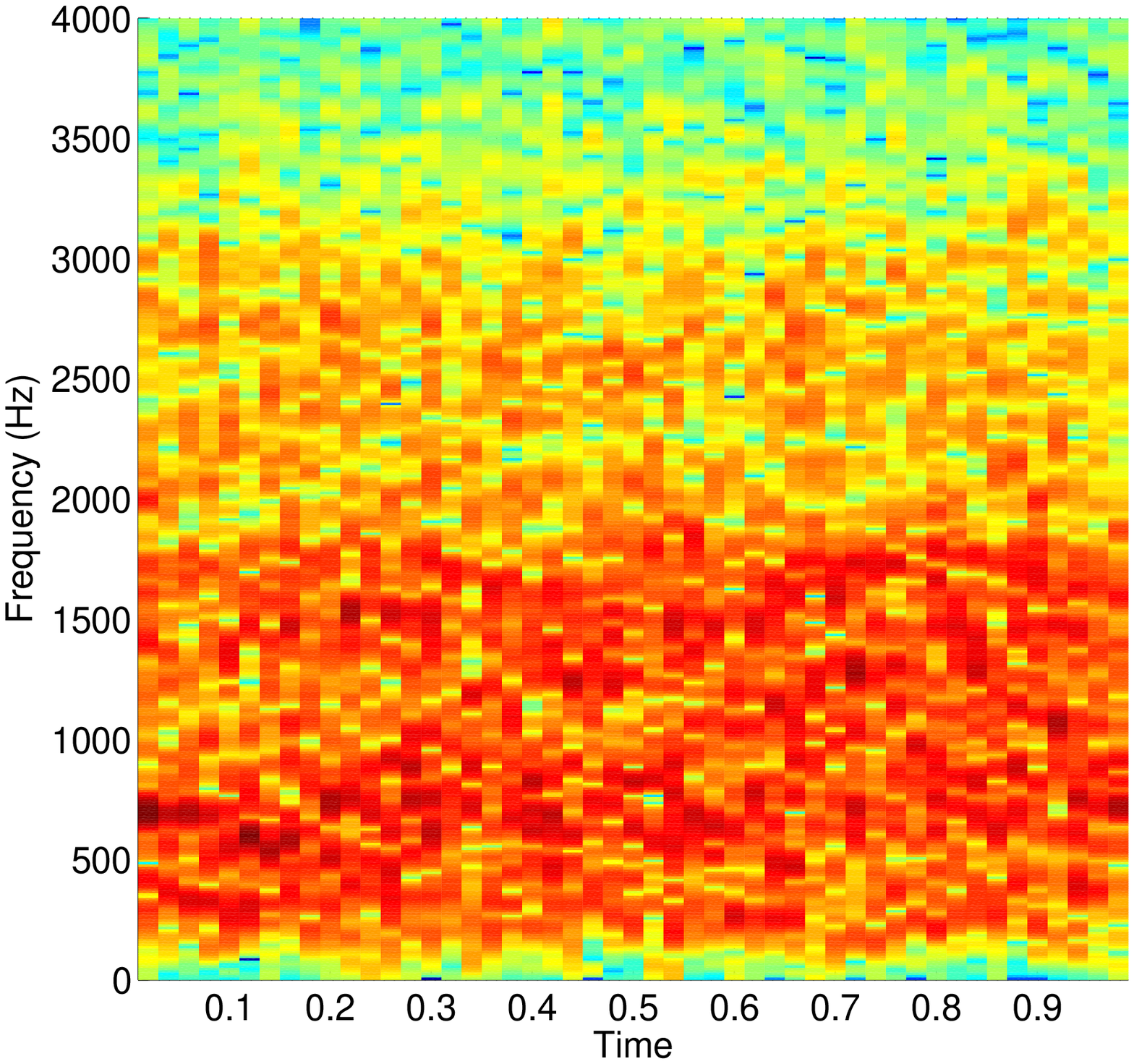}
}
\subfigure[]{
\includegraphics[width=0.4\textwidth]{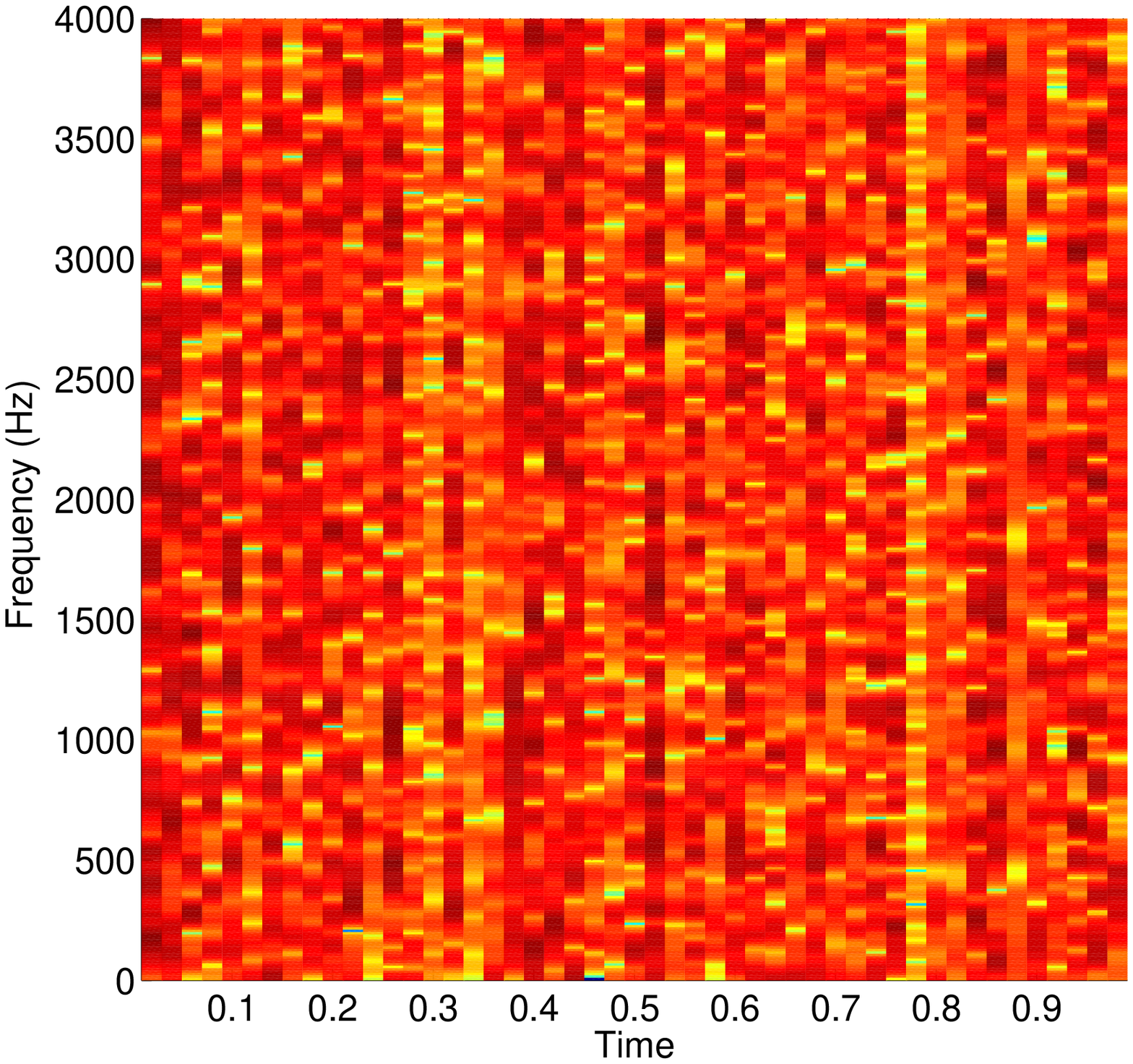}
}
\hspace{1in}
\subfigure[]{
\includegraphics[width=0.4\textwidth]{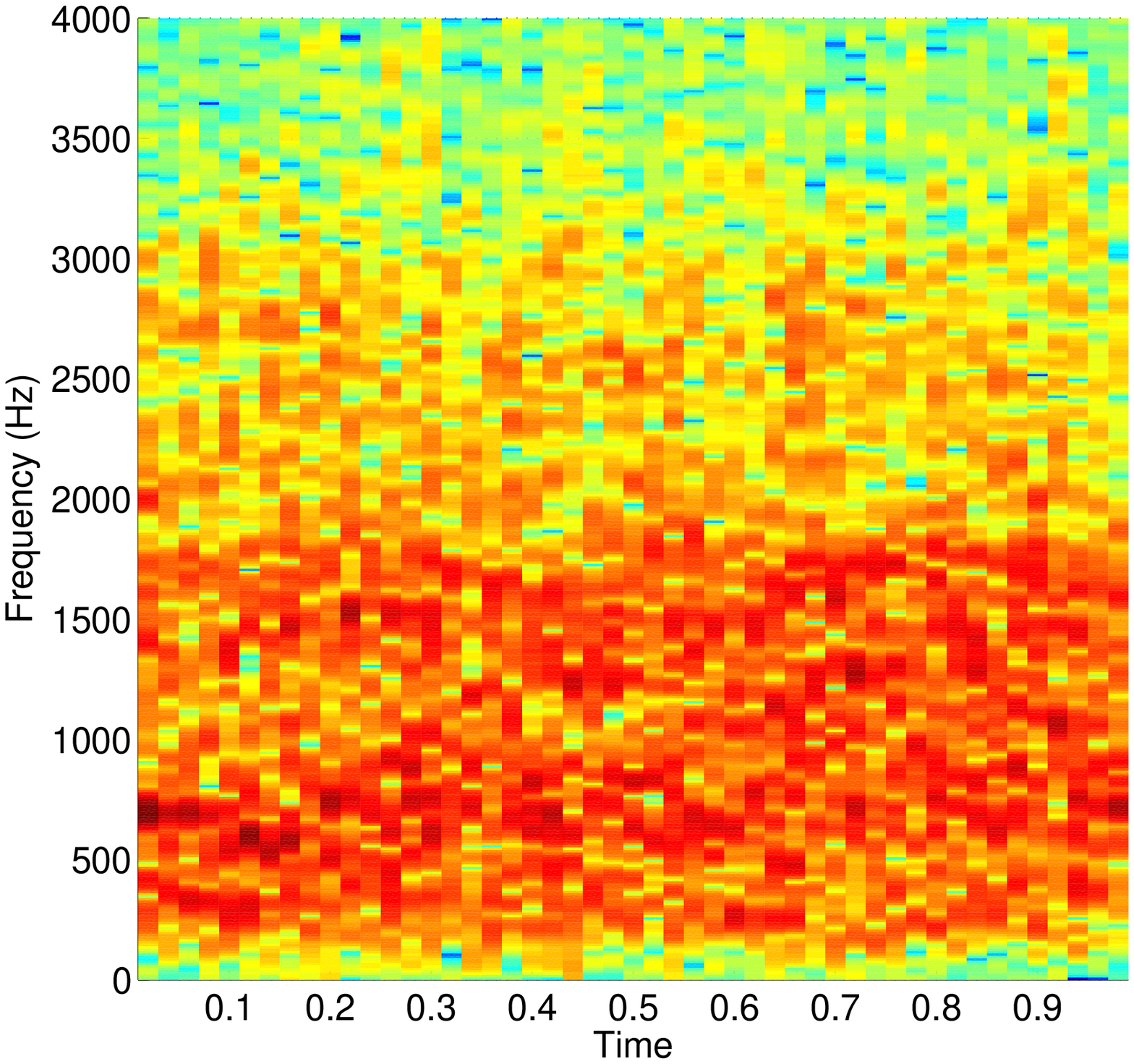}
}
\caption{spectrograms of low-rank approximations of the audio segments
with or without noise with $\lambda = 1$ throughout. (a) Spectrograms
of the original laugh segment; (b) Spectrograms of the low-rank approximation of the laugh segment; (c) Spectrograms of the same audio segment corrupted
by white Gaussian noise with SNR=20dB; (d) Spectrograms of the low-rank approximation of the laugh segment with white Gaussian noise; (e) Spectrograms of of the same audio segment corrupted
by white Gaussian noise and random large errors; (f) Spectrograms of the low-rank
approximation of the laugh segment white Gaussian noise and random large errors.}
\label{fig:Spectrograms_RPCAtoAudioSeg}       
\end{figure}

\section{Low Rank Matrix Classification}
\label{sec:MatrixClassifyAED}
\subsection{Notation and Problem Statement}
Having extracted robust matrix representation features, the linear matrix classification approach based on trace norm regularization framework proposed in~\cite{Tomioka2007} is used to classify them. The motivation for trace norm regularization framework is two fold: a) trace norm considers the interactive information among the frames in the matrix while the simple approach that treat
the matrix as a long vector would lose the information; b) trace norm is a suitable quantity that measures the complexity of the linear classifier. Generally, the problem for trace norm regularization based matrix classification is formulated as
\begin{equation}\label{eq:PrblmFrmltn}
\mathop {\min}\limits_{W,b}F_s (W,b)=f_s(W,b)+\lambda\left\|W\right\|_*
\end{equation}
where $W\in \mathbb{R}^{m\times n}$ is the unknown \emph{weight matrix}, $b\in \mathbb{R}$ is the \emph{bias}, $\left\|\cdot\right\|_*$ denotes the trace norm defined as the sum of the singular values, and $\lambda$ is the regularization parameter. $f_s(W,b)=\sum\limits_{i=1}^s {\ell(y_i, \mathrm{Tr}(W^{T}X_i)+b)}$ is the empirical cost function induced by some convex smooth loss function $\ell(\cdot,\cdot)$, where $\mathrm{Tr}(\cdot)$ denotes the trace, the subscript of $f_s(W,b)$ indicates the number of training samples or time of training procedure which is apparent from context, and $(X_i,y_i)\in \mathbb{R}^{m\times n}\times \mathbb{R}$ is the $i$th sample. In this work, the standard squared loss function is used. Hence the empirical cost function becomes $f_s(W,b)=\sum\limits_{i=1}^s {(y_i-\mathrm{Tr}(W^{T}X_i)-b)^2}$.

\subsection{APG Method for Matrix Classification}

Recently Toh and Yun~\cite{Toh2010}, Ji and Ye~\cite{Ji2009}, and Liu et al.~\cite{Liu2009} independently proposed similar algorithms that converge as $O(\frac{1}{k^2})$ for problem~(\ref{eq:PrblmFrmltn}) by using APG, where $k$ is the iteration counter. The precondition of using APG algorithm is that the loss function should be smooth, convex, and the gradient should satisfy Lipschitz condition. Since $f_s(W,b)$ in this work is a composition of smooth convex function with an affine mapping, hence it is convex and smooth~\cite{boyd2004convex}. For Lipschitz continuous, it is shown  in Theorem~\ref{theorem:Lipschitz} that the gradient of $f_s(W,b)$, denoted as
\begin{equation}\label{eq:Gradient}
\nabla_W f_s(W,b)=-2\sum\limits_{i=1}^s {(y_i-\mathrm{Tr}(W^{T}X_i)-b)X_i},
\end{equation}
is Lipschitz continuous. Thus the APG method can be used to solve matrix classification problem. In order to solve the unconstrained convex optimization problem~(\ref{eq:PrblmFrmltn}), APG approximate $f_s(W,b)$ locally as a quadratic function with bias fixed and solve
\begin{equation}\label{eq:LocallyApproximate}
\begin{split}
W_{k+1}=\mbox{arg}\!\min_{W\in\mathbb{R}^{m\times n}} Q(W,Z_k)
=f_s(Z_k,b)+\frac{t_k}{2}\left\|W-Z_k\right\|_F^2 \\
+<\nabla_W f_s(Z_k,b),W-Z_k>+\lambda\left\|W\right\|_*,
\end{split}
\end{equation}
which is assumed to be easy, to update the solution $W$. Based on the the work of Nesterov~\cite{Nesterov1983,Nesterov2005}, Toh and Yun~\cite{Toh2010}, Ji and Ye~\cite{Ji2009}, and Liu et al.~\cite{Liu2009} showed that setting $Z_k=W_k+\frac{t_{k-1}-1}{t_k}(W_k-W_{k-1})$ for a sequence ${t_k}$ satisfying $t_{k+1}^2-t_{k+1}\leq t_{k}^2$ results in a convergence rate of $O(\frac{1}{k_2})$. Due to Lemma~\ref{theorem:Lipschitz}, the estimation of step size $t_k$ in general APG~\cite{Toh2010,Ji2009,Liu2009} is omitted, for we have explicit Lipschitz constant. The APG approach for batch-mode weight matrix learning is described in Algorithm~\ref{algo:APG}. The $\mathcal{S}_{\varepsilon}[\cdot]$ in Algorithm~\ref{algo:APG} is the soft-thresholding operator introduced in~\cite{Lin2010}:
\begin{equation}
\mathcal{S}_{\varepsilon}[x]\doteq \left\{ \begin{array}{l}
 x-\varepsilon, \textrm{if } x>\varepsilon,   \\
 x+\varepsilon, \textrm{if } x<-\varepsilon,   \\
0, \textrm{otherwise}   \\
 \end{array} \right.
\end{equation}
where $x\in \mathbb{R}$ and $\varepsilon > 0$. For vectors and matrices, this operator is extended by applying element-wise.

\begin{algorithm}
\label{algo:APG}
Batch-Mode Weight Matrix Learning via APG

\textbf{Initialize} $W_0=Z_1\in\mathbb{R}^{m\times n},\alpha_1 =1, L=2mn\sum\limits_{i=1}^s \left\|X_i\right\|_F^2, \lambda. $

1: \textbf{while} not converged \textbf{do}

2: $(U,S,V)=\textrm{svd}(Z_k-\frac{1}{L}(-2\sum\nolimits_{i=1}^s {(y_i-\mathrm{Tr}(Z_k^{T}X_i)-b)X_i}))$.

3: $W_k=U\mathcal{S}_{\frac{\lambda}{L}}[S]V^T$.

4: $\alpha_{k+1}=\frac{1+\sqrt{1+4\alpha_k^2}}{2}$.

5: $Z_{k+1}=W_k+\frac{\alpha_{k}-1}{\alpha_{k+1}}(W_k-W_{k-1})$.

6: $b_k=\frac{1}{s}\sum\limits_{i=1}^s (y_i-\mathrm{Tr}(W^{T}_{k}X_i)).$

7: $k\leftarrow k+1$.

8: \textbf{end while}

\textbf{Output}: $W\leftarrow W_k$.
\end{algorithm}

The general APG~\cite{Toh2010,Ji2009,Liu2009} algorithms only provide the methods for learning weight matrices, do not give out the bias updating rules. In order to update the bias $b$, fixes the weight matrix $W_k$ and solve the following problem
\begin{equation}\label{eq:BiasUpdateProblem}
b_k = \mathop {\min}\limits_{b}\sum\limits_{i=1}^s {(y_i-\mathrm{Tr}(W^{T}_{k}X_i)-b)^2}+\lambda\left\|W_k\right\|_*,
\end{equation}
which results in the bias updating rule
\begin{equation}\label{eq:BiasUpdateRule}
b_k = \frac{1}{s}\sum\limits_{i=1}^s (y_i-\mathrm{Tr}(W^{T}_{k}X_i)).
\end{equation}
This results in the line 6 of Algorithm~\ref{algo:APG}.
For the stopping criteria of the iterations, we take the following relative error conditions:
\begin{equation}\label{eq:StoppingCriteria}
\|W_{k+1}-W_{k}\|_F/\|W_{k}\|_F<\varepsilon_1\text{ and }|b_{k+1}-b_{k}|/|b_{k}|<\varepsilon_2.
\end{equation}

After the weight matrix $W$ and bias $b$ are found, the observed MFCCs matrix $X_i$ can be classified via
\begin{equation}\label{eq:PredictMFCCmatric}
\hat{y}_i = \mathrm{Tr}(W^{T}X_i) + b.
\end{equation}


\subsection{Determination of Lipschitz Constant}
As a special case of general convex optimization problem, we derived the closed-form of the Lipschitz constant, hence the step size estimation~\cite{Toh2010,Ji2009} of the general APG method was omitted in all our approach. The determination of the Lipschitz constant is shown in the following theorem.

\begin{theorem}
\label{theorem:Lipschitz}
$\nabla_W f_s(\cdot,b)$ is Lipschitz continuous with constant $L=2mn\sum\limits_{i=1}^s \left\|X_i\right\|_F^2$, i.e., $\forall U, V\in \mathbb{R}^{m\times n}$,
\begin{equation}\label{eq:LipschitzCondition}
\left\|\nabla_W f_s(U,b)-\nabla_W f_s(V,b)\right\|_F\leq L\left\|U-V\right\|_F,
\end{equation}
where $\left\|\cdot\right\|_F$ denotes the Frobenius norm.
\end{theorem}

\begin{proof}
Applying Equation~(\ref{eq:Gradient}) with $U,V$ to the right of Equation~(\ref{eq:LipschitzCondition}),
we obtain
\begin{equation}
\begin{split}
 &\left\|\nabla_W f_s(U,b)-\nabla_W f_s(V,b)\right\|_F \nonumber \\
=& \| -2\sum\nolimits_{i=1}^s {(y_i-\mathrm{Tr}(U^{T}X_i)-b)X_i} \\
& +2\sum\nolimits_{i=1}^s {(y_i-\mathrm{Tr}(V^{T}X_i)-b)X_i} \|_F \\
=& 2\left\| \sum\nolimits_{i=1}^s {(\mathrm{Tr}(U^{T}X_i)l-\mathrm{Tr}(V^{T}X_i))X_i}\right\|_F \\
\le& 2\sum\nolimits_{i=1}^s \left| \mathrm{Tr}((U^{T}-V^{T})X_i)\right| \left\| X_i\right\|_F \\
\le& 2mn\sum\nolimits_{i=1}^s \left\| U^{T}-V^{T}\right\|_F \left\| X_i\right\|_F^2 \\
=& (2mn\sum\nolimits_{i=1}^s  \left\| X_i\right\|_F^2)\left\| U^{T}-V^{T}\right\|_F,
\end{split}
\end{equation}
where in the last inequality, the easily verified fact that $\mathrm{Tr}(A^{T}B)\le \left\| A\right\|_1\left\| B \right\|_1 \le mn\left\| A\right\|_F\left\| B \right\|_F$ for
$\forall A,B \in \mathbb{R}^{m\times n}$ is used. Here $\left\| \cdot \right\|_1$ denotes the $\ell_1$ norm which is the sum of the absolute values of the matrix elements.

Thus the lemma is proofed, that is to say $\nabla_W f_s(\cdot,b)$ is Lipschitz continuous with constant $L=2mn\sum\nolimits_{i=1}^s \left\|X_i\right\|_F^2$.
\end{proof}

The APG based batch-mode weight learning method is effective for small training set, but with large training sets, this classical optimization technique may become impractical in terms of memory requirements. Furthermore, this method cannot efficiently deal with dynamic training data of time sequences, such as audio and video processing. To tackle the insufficiency, we propose an online learning framework in the following section.


\section{Online Learning for Matrix Classification}
\label{sec:OnlineLearning}
\subsection{Online Learning with APG}
We present in this section the basic components of our online learning algorithm for matrix classification, as well as a few minor variants which speed up our implementation in practice.

\begin{algorithm}
\label{algo:OnlineLearning}
Online MC Learning Based on APG.

\textbf{Initialize} $W_0\in\mathbb{R}^{m\times n}, b_0\in\mathbb{R}, L_0=0, \lambda\in \mathbb{R}.$

1: $A_0\in \mathbb{R}^{m\times n}\leftarrow 0, B_0\in\mathbb{R}^{mm\times nn}\leftarrow 0, c_0\in \mathbb{R}\leftarrow 0, D_0\in \mathbb{R}^{m\times n}\leftarrow 0$(reset the ``past'' information).

2: \textbf{for} $t=1$ \textbf{to} $T$ \textbf{do}

3: Draw training sample $(X_t,y_t)$ from $p(X,y)$.

4: // Line 5-9 update ``past'' information.

5: $A_t\leftarrow A_{t-1}+y_tX_t$;

6: $B_t\leftarrow B_{t-1}+X_t\otimes X_t$;

7: $c_t\leftarrow c_{t-1}+y_t$;

8: $D_t\leftarrow D_{t-1}+X_t$;

9: $L_t\leftarrow L_{t-1}+2mn\left\|X_t\right\|_F^2$.

10: // Line 11-19 update $W_t$ and $b_t$ using Algorithm~\ref{algo:APG}, with $W_{t-1}$ and $b_{t-1}$ as warm restart.

11: $W_{0,t}=Z_{1,t}=W_{t-1}\in\mathbb{R}^{m\times n},b_{0,t}=b_{t-1},\alpha_1 =1,k=1.$

12: \textbf{while} not converged \textbf{do}

13: $(U,S,V)=\textrm{svd}(Z_{k,t}-\frac{1}{L_t}(-2A_t+2\textrm{GridTr}(Z_{k,t},B_t)+2b_{k-1,t}D_t)$.

14: $W_{k,t}=U\mathcal{S}_{\frac{\lambda}{L_t}}[S]V^T$.

15: $\alpha_{k+1}=\frac{1+\sqrt{1+4\alpha_k^2}}{2}$.

16: $Z_{k+1,t}=W_{k,t}+\frac{\alpha_{k}-1}{\alpha_{k+1}}(W_{k,t}-W_{k-1,t})$.

17: $b_{k,t}=\frac{1}{t}(c_t-\mathrm{Tr}(W^{T}_{k,t}D_t)$

18: $k\leftarrow k+1$.

19: \textbf{end while}

20: $W_t\leftarrow W_{k,t},b_t\leftarrow b_{k,t}.$

21: \textbf{end for}

\textbf{Output}: $W\leftarrow W_T, b\leftarrow b_T.$
\end{algorithm}


Our procedure is summarized in Algorithm~\ref{algo:OnlineLearning}. The $\otimes$ operator in step 6 of the algorithm denotes the Kronecker product. Given two matrices $A\in \mathbb{R}^{m_1\times n_1}$ and $B\in \mathbb{R}^{m_2\times n_2}$, $A\otimes B$ denotes the Kronecker product between $A$ and $B$, defined as the matrix in $\mathbb{R}^{m_1m_2\times n_1n_2}$, defined by blocks of sizes $m_2\times n_2$ equal to $A[i,j]B$. $\textrm{GridTr}(Z_{k,t},B_t)$ in step 13 denotes an operator with input $Z_{k,t}\in \mathbb{R}^{m\times n}$ and $B_t\in \mathbb{R}^{mm\times nn}$, result in $\mathbb{R}^{m\times n}$ with the $(i,j)$th element defined as the trace of the product between $Z_{k,t}^T$ and the $(i,j)$th $\mathbb{R}^{m\times n}$ block of $B_t$.

Assuming the training set composed of i.i.d. samples of a distribution $p(X,y)$, its inner loop draws one training sample $(X_t,y_t)$ at a time. This sample is first used to update the ``past'' information $A_{t-1}$, $B_{t-1}$, $c_{t-1}$, and $D_{t-1}$. Then the Algorithm~\ref{algo:APG} is applied to update the weight matrix with the warm start $W_{t-1}$ obtained at the previous iteration. Since $F_t(W,b_{t-1})$ is relative close to $F_{t-1}(W,b_{t-1})$ for large values of $t$, so are $W_t$ and $W_{t-1}$, under suitable assumptions, which makes it efficient to use $W_{t-1}$ as warm restart for computing $W_t$.

\subsection{Online Learning with inexact APG}

Algorithm~\ref{algo:OnlineLearning} calls APG to update the weight matrix for each coming sample by solving the sub-problem with fixed bias $b$
\begin{equation}\label{eq:SubProblem}
W_t = \mathop {\min}\limits_{W}\sum\limits_{i=1}^t {(y_i-\mathrm{Tr}(W^{T}X_i)-b_{t-1})^2}+\lambda\left\|W\right\|_*
\end{equation}
exactly which cause computational load for large scale training set. Fortunately, due to the closeness of consecutive weight matrix, we do not have to solve the sub-problem exactly. Rather, updating $W_{t-1}$ once when solving this sub-problem is sufficient in practice. This leads to an online MC learning method based on inexact APG, described in Algorithm~\ref{algo:InexactOnlineLearning}.

\begin{algorithm}
\label{algo:InexactOnlineLearning}
Online MC Learning with Inexact APG.

\textbf{Initialize} $W_0\in\mathbb{R}^{m\times n}, b_0\in\mathbb{R}, L_0=0, \lambda\in \mathbb{R}.$

1: $A_0\in \mathbb{R}^{m\times n}\leftarrow 0, B_0\in\mathbb{R}^{mm\times nn}\leftarrow 0, c_0\in \mathbb{R}\leftarrow 0, D_0\in \mathbb{R}^{m\times n}\leftarrow 0$ (reset the ``past'' information).

2: \textbf{for} $t=1$ \textbf{to} $T$ \textbf{do}

3: Draw training sample $(X_t,y_t)$ from $p(X,y)$.

4: // Line 5-9 update ``past'' information.

5: $A_t\leftarrow A_{t-1}+y_tX_t$;

6: $B_t\leftarrow B_{t-1}+X_t\otimes X_t$.

7: $c_t\leftarrow c_{t-1}+y_t$;

8: $D_t\leftarrow D_{t-1}+X_t$.

9: $L_t\leftarrow L_{t-1}+2mn\left\|X_t\right\|_F^2$.

10: // Line 11-16 compute $W_t$ using inexact APG, with $W_{t-1}$ as warm restart.

11: $W_{0,t}=W_{t-1}\in\mathbb{R}^{m\times n}.$

12: $(U,S,V)=\textrm{svd}(W_{0,t}-\frac{1}{L_t}(-2A_t+2\textrm{GridTr}(W_{0,t},B_t)+2b_{t-1}D_t)$.

13: $W_{1,t}=U\mathcal{S}_{\frac{\lambda}{L_t}}[S]V^T$.

14: $(U,S,V)=\textrm{svd}(W_{1,t}-\frac{1}{L_t}(-2A_t+2\textrm{GridTr}(W_{1,t},B_t)+2b_{t-1}D_t)$.

15: $W_{2,t}=U\mathcal{S}_{\frac{\lambda}{L_t}}[S]V^T$.

16: $W_t\leftarrow W_{2,t}.$

17: // Line 18 updates the bias $b_t$.

18: $b_t=\frac{1}{t}(c_t-\mathrm{Tr}(W^{T}_{t}D_t)$

19: \textbf{end for}

\textbf{Output}: $W\leftarrow W_T, b\leftarrow b_T.$
\end{algorithm}

\subsection{Online Learning with Mini-batch}

In some conditions, use the classical heuristic in gradient descent algorithm, we may also improve the convergence speed of our algorithm by drawing $\mu > 1$ training samples at each iteration instead of a single one. Let us denote by $(X_{t,1},y_{t,1}),...,(X_{t,\mu},y_{t,\mu})$ the samples drawn at iteration $t$. We can now replace lines 5 and 9 of Algorithm~\ref{algo:OnlineLearning} and \ref{algo:InexactOnlineLearning} by
\begin{equation}\label{eq:MiniBatchInfoUpdata}
\begin{array}{l}
A_t\leftarrow A_{t-1}+\sum\limits_{i=1}^{\mu}{y_{t,i}X_{t,i}}, \\
B_t\leftarrow B_{t-1}+\sum\limits_{i=1}^{\mu}{X_{t,i}\otimes X_{t,i}}, \\
c_t\leftarrow c_{t-1}+\sum\limits_{i=1}^{\mu}{y_{t,i}}, \\
D_t\leftarrow D_{t-1}+\sum\limits_{i=1}^{\mu}{X_{t,i}}, \\
L_t\leftarrow L_{t-1}+\sum\limits_{i=1}^{\mu}{2mn\left\|X_{t,i}\right\|_F^2}.
\end{array}
\end{equation}
But in real applications, this batch method may not improve the convergence speed on the whole since the batch past information computation (Equation~(\ref{eq:MiniBatchInfoUpdata})) would occupy much of the time. The updating of $B_t$ needs to do Kronecher product which spend much of the computing resource. If the computation cost of Equation~(\ref{eq:MiniBatchInfoUpdata}) can be ignored or largely decreased, for example by parallel computing, the batch method would increase the convergence speed by a factor of $\mu$.

\section{Experimental Validation}
\label{sec:Experimental}
\subsection{Dataset}
\label{sec:DatasetandExpeSetup}
Experiments are conducted on a collected database. We downloaded about 20hours videos from Youku~\cite{Youku}, with different programs and different languages. The start and end position of all the applause and laugh of the audio-tracks are manually labeled. The database includes 800 segments of each sound effect. Each segment is about 3-8s long and totally about 1hour data for each sound effect. All the audio recordings were converted to monaural wave format at a sampling frequency of 8kHz and quantized 16bits. Furthermore, the audio signals have been normalized, so that they have zero mean amplitude with unit variance in order to remove any factors related to the recording conditions.

\subsection{Online Learning}
\label{sec:OnlineLearning}
In this section, we conduct detailed experiments to demonstrate the characteristics and merits of the online learning for matrix classification problem. Five algorithms are compared: the traditional batch algorithm with exact APG algorithm (APG); the online learning algorithm with exact APG (OL\_APG); the online learning algorithm with inexact APG (OL\_IAPG); the online learning algorithm with exact APG and update Equation~(\ref{eq:MiniBatchInfoUpdata}) (OL\_APG\_Batch); the online learning algorithm with inexact APG and update Equation~(\ref{eq:MiniBatchInfoUpdata}) (OL\_IAPG\_Batch). All algorithms are run in Matlab on a personal computer with an Intel 3.40GHz dual-core central processing unit (CPU) and 2GB memory.

For this experiment, audio streams were windowed into a
sequence of short-term frames (20 ms long) with non overlap. 13 dimensional
MFCCs including energy are extracted, and adjacent 50 frames (one second) of
MFCCs form the MFCCs matrix feature. The goal is to classify the matrices
according to their labels. Two learning tasks are used to evaluate the
performance of the online learning method, which are laugh/non-laugh segment
classifier learning and applause/non-applause segment classifier learning.
For OL\_APG and OL\_APG\_Batch algorithms, the parameters in the stopping
criteria~(\ref{eq:StoppingCriteria}) are set $\varepsilon_1=10^{-8}$ and
$\varepsilon_2=10^{-8}$ or smaller, which are determined by empirical evidence that larger values would make the algorithm diverge. The regularization constant $\lambda$
is anchored by the large explicit fixed step size $L$ and the matrices involved, this can be seen from $\frac{\lambda}{L}$ in the line 3 in Algorithm~\ref{algo:APG}, which means that in practice the parameter $\lambda$ should be set adaptably with the step size $L$ in the online process. But due to this variation of $\lambda$, the comparisons between the algorithms would not bring into effect. Hence in this work we use $\lambda = 1$ throughout.
\begin{figure*}
\subfigure[]{
\includegraphics[width=0.45\textwidth]{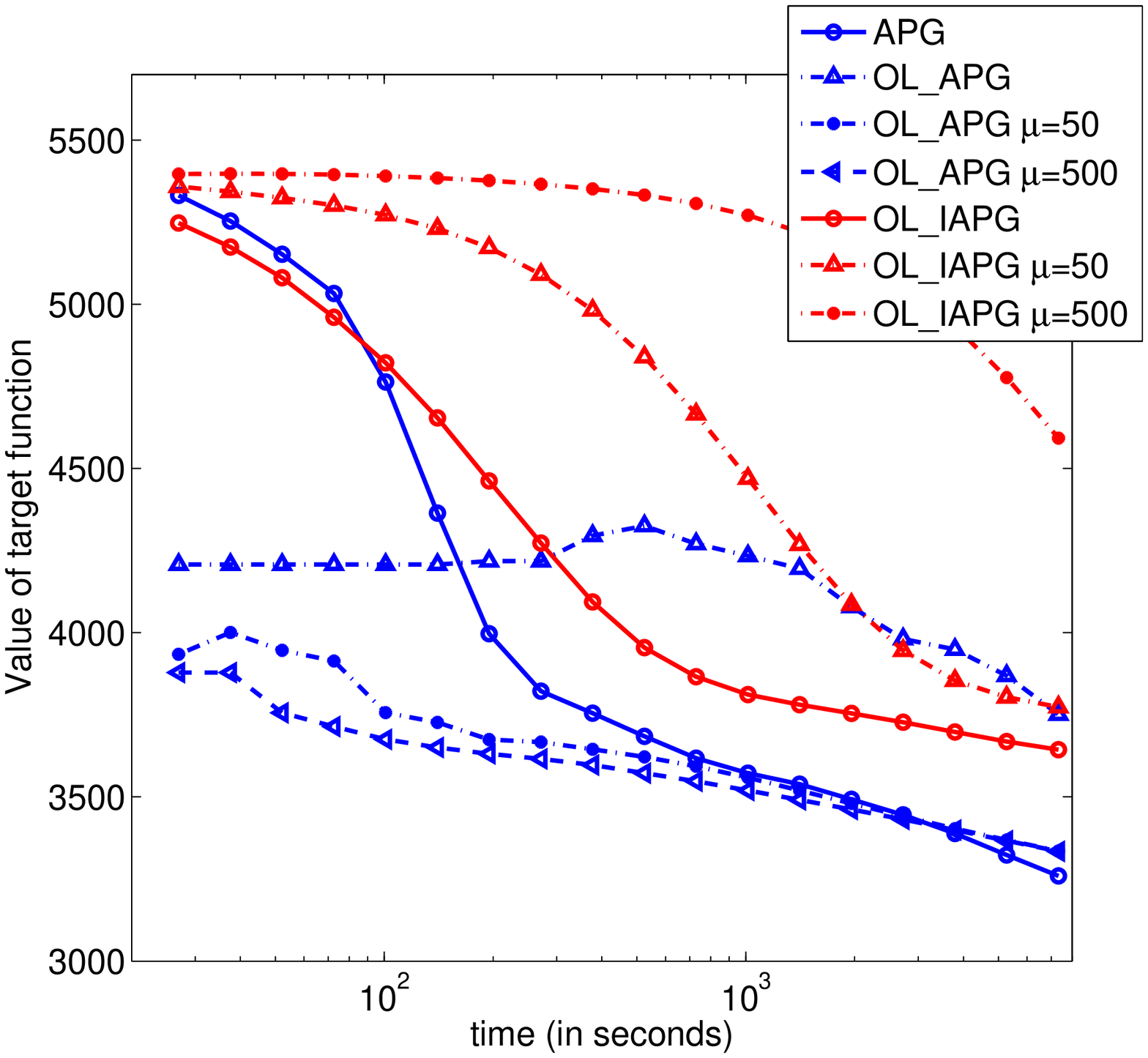}
}
\subfigure[]{
\includegraphics[width=0.45\textwidth]{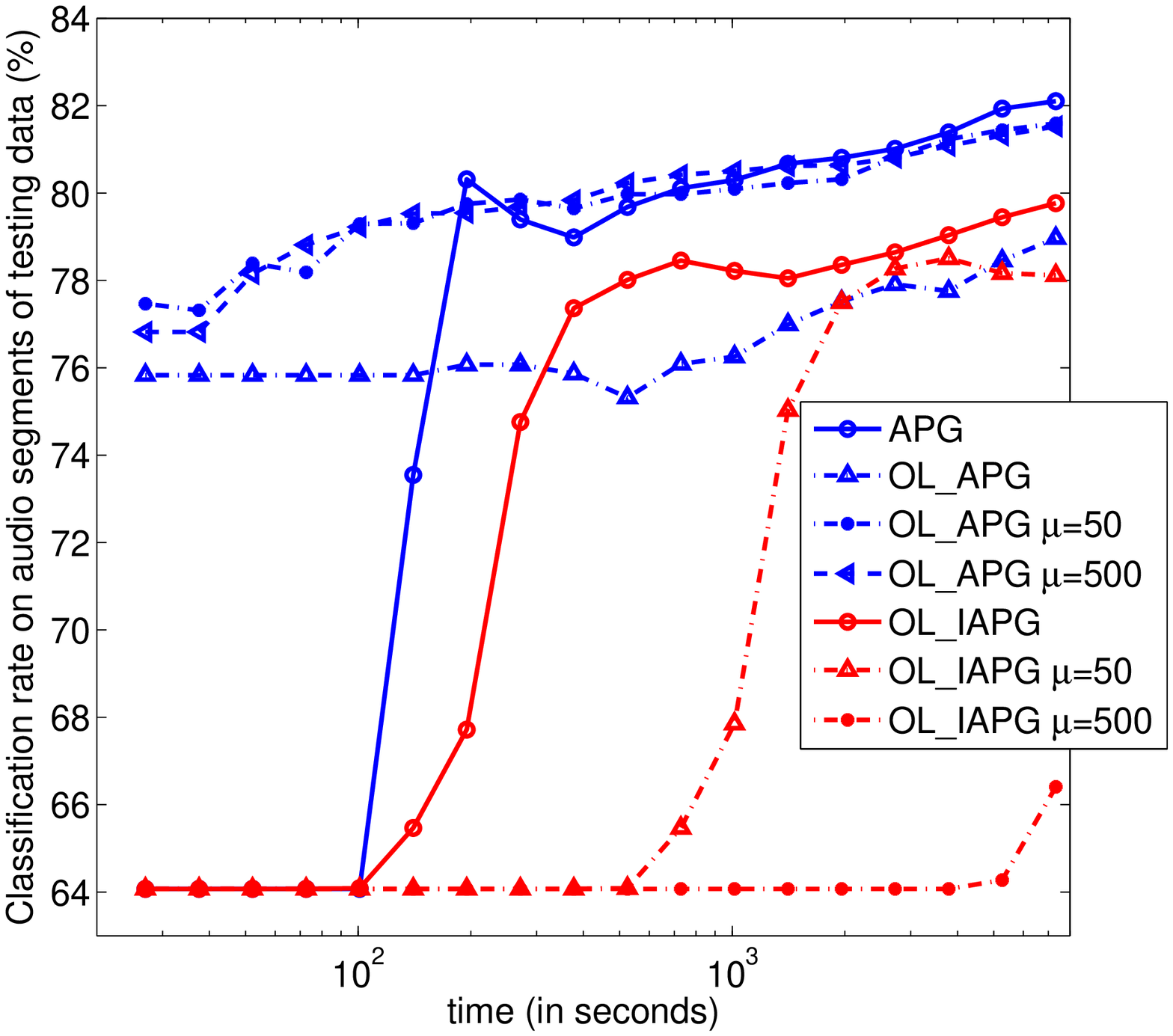}
}
\subfigure[]{
\includegraphics[width=0.45\textwidth]{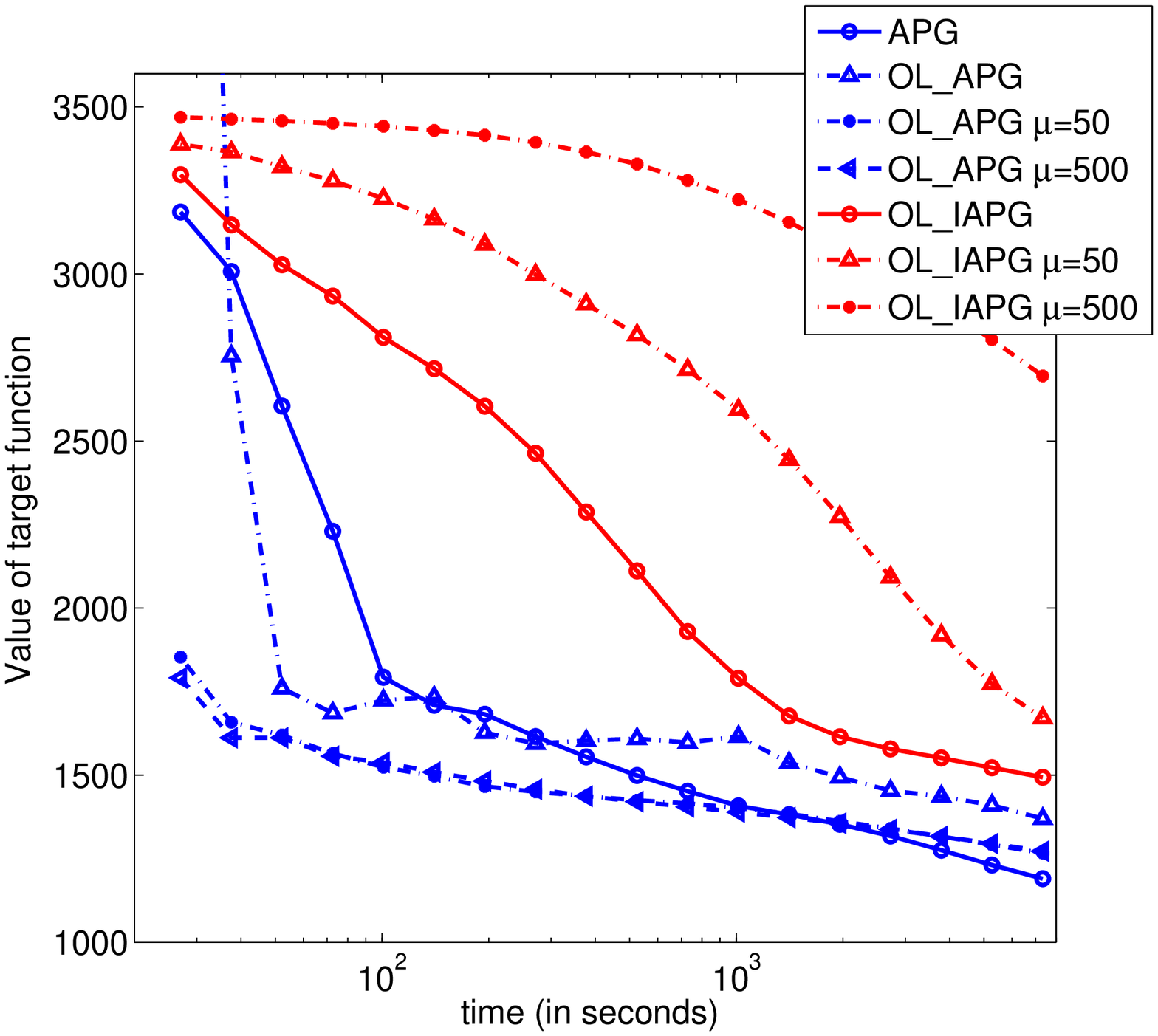}
}
\hspace{0.5in}
\subfigure[]{
\includegraphics[width=0.45\textwidth]{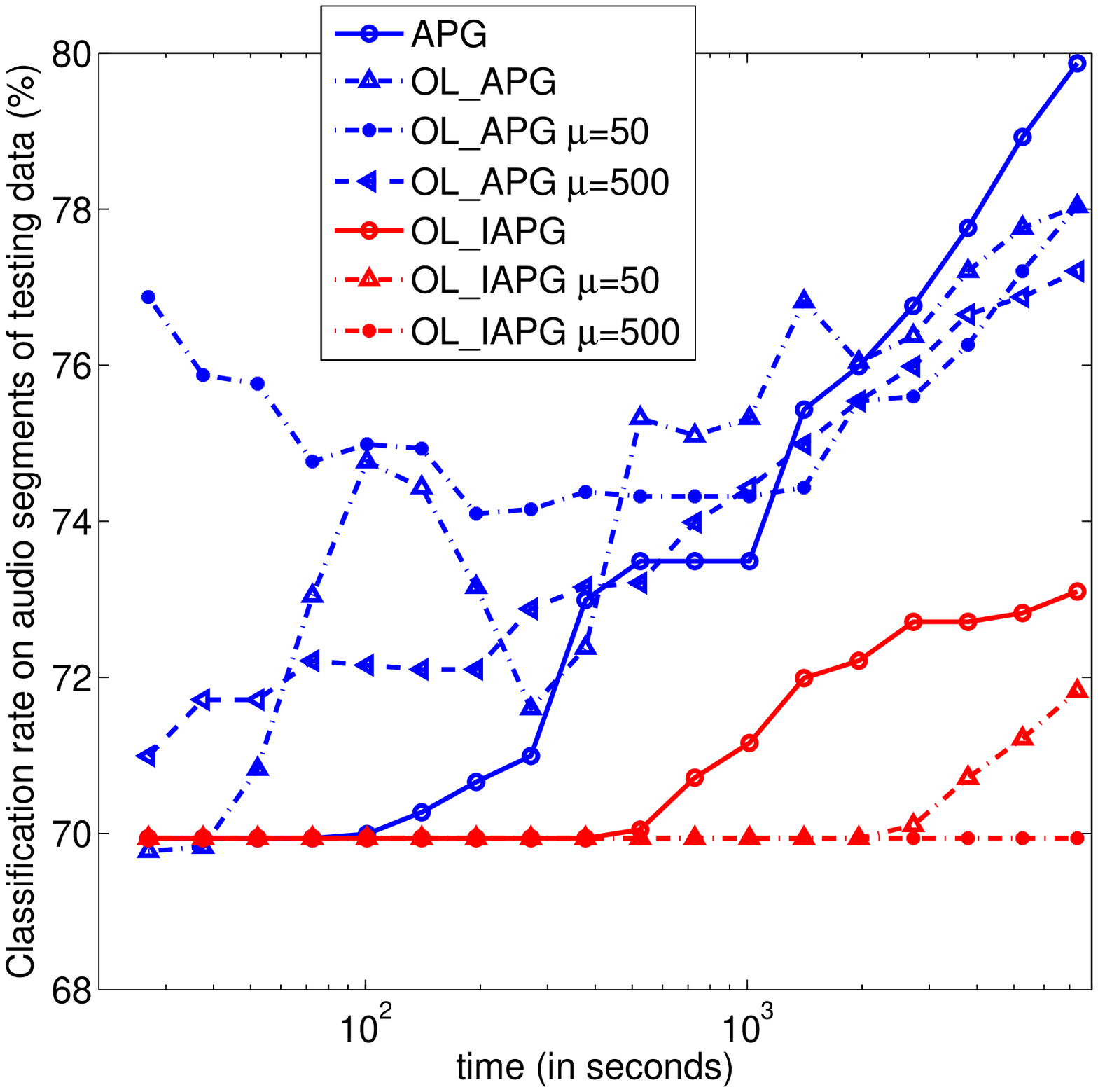}
}
\caption{Comparisons between various online learning methods and results are
reported as functions of learning time on a logarithmic scale. (a) Value of
target function for online learning of applause segments classifier; (b)
Classification rate on audio segments of testing data for online learning
of applause segments classifier; (c) Value of target function for online
learning of applause segments classifier; (d) Classification rate on audio
segments of testing data for online learning of laugh segments classifier.}
\label{fig:PerformancComparOnline}       
\end{figure*}

Fig.~\ref{fig:PerformancComparOnline} compares the five online algorithms. The proposed online algorithm draws samples from the entire training set. We use a logarithmic scale for the computation time. Fig.~\ref{fig:PerformancComparOnline}a shows the values of the target functions as functions of time. It can be seen that the online learning methods without batch or with small batch past information updating converge faster than the methods with large batch past information updating and reason for this has been explained in the last paragraph of Section~\ref{sec:OnlineLearning}. After online methods and batch methods converge, the two methods result in almost equal performance. Fig.~\ref{fig:PerformancComparOnline}(b)(d) shows the classification rates for different algorithms respectively. In accordance with the values of the target functions, the classification accuracies of online methods without or with small batch updating become stable quickly than that of methods with batch updating. Although the inexact algorithms process samples much fast with less resources than exact ones, they converge slowly.

\subsection{Robustness}

 This section is to assess the effectiveness of robust PCA extracted low-rank matrix features. Original features (MFCCs\_Matrix), corrupted with 0dB and -5dB white Gaussian noise (WGN SNR=5dB, 0dB, -5dB) and 10\%, 30\%, 50\% random large errors (LE 10\%, 30\%, 50\%), and parallelism robust PCA extracted features (rPCA) are compared. In the comparisons, the parameters in the stopping criteria~(\ref{eq:StoppingCriteria}) are set $\varepsilon_1=10^{-6}$ and $\varepsilon_2=10^{-6}$, which are determined by the same method as in Section~\ref{sec:OnlineLearning}. The regularization constant $\lambda$ is set $1/\sqrt{50}$ which is a classical normalization factor according to~\cite{Bickel}.

\begin{figure*}
\subfigure[]{
\includegraphics[width=0.45\textwidth]{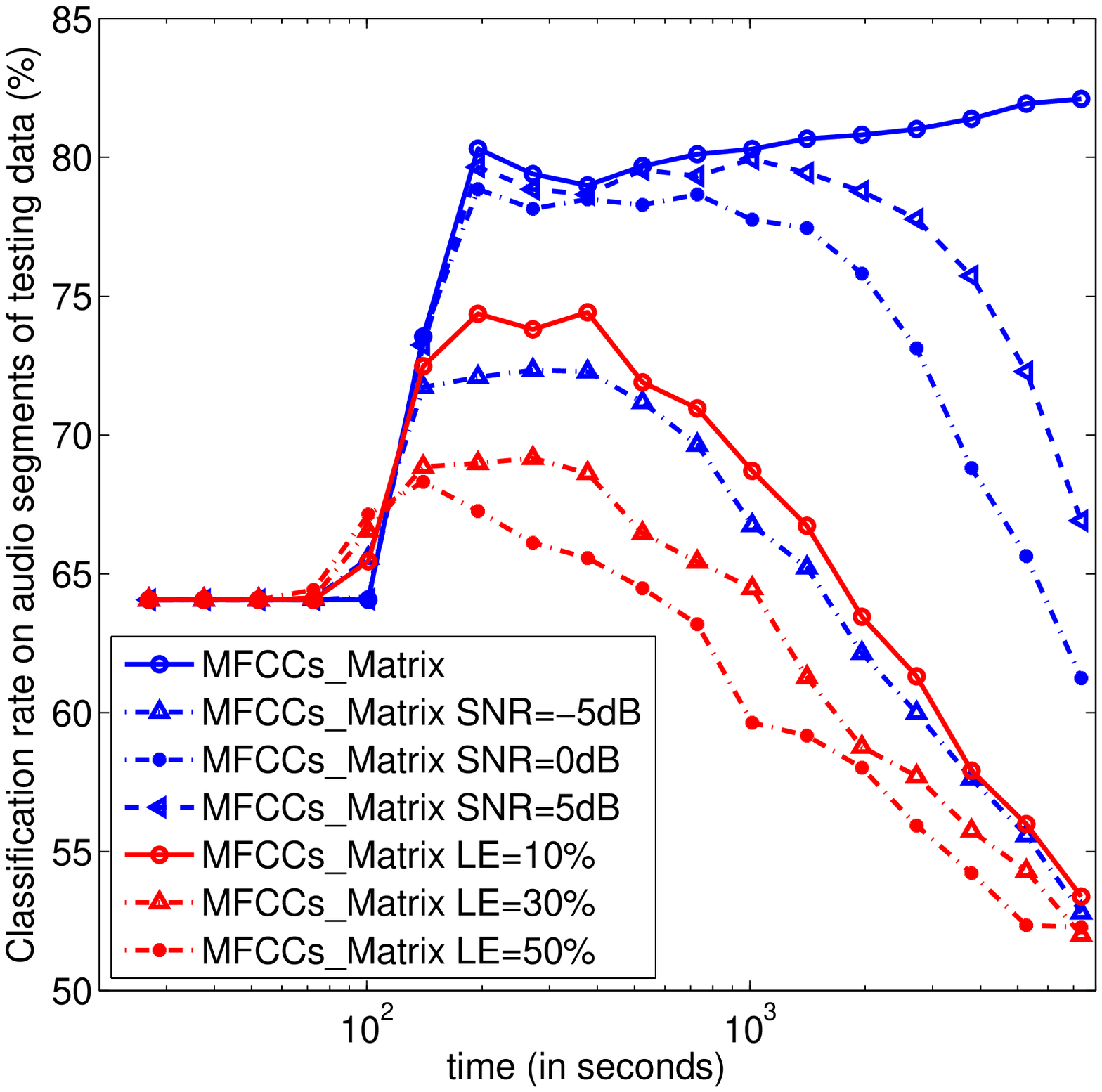}
}
\subfigure[]{
\includegraphics[width=0.45\textwidth]{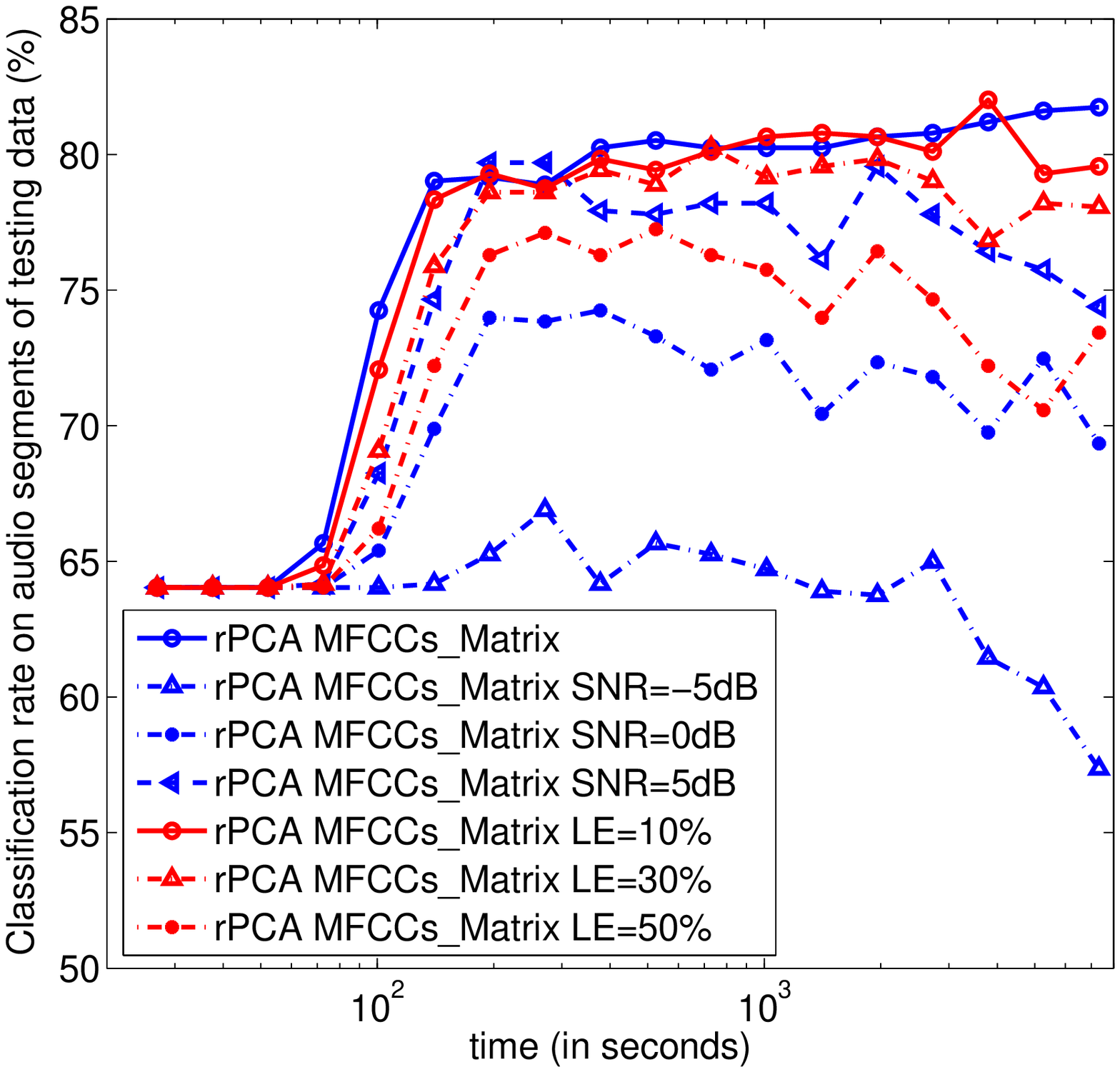}
}
\subfigure[]{
\includegraphics[width=0.45\textwidth]{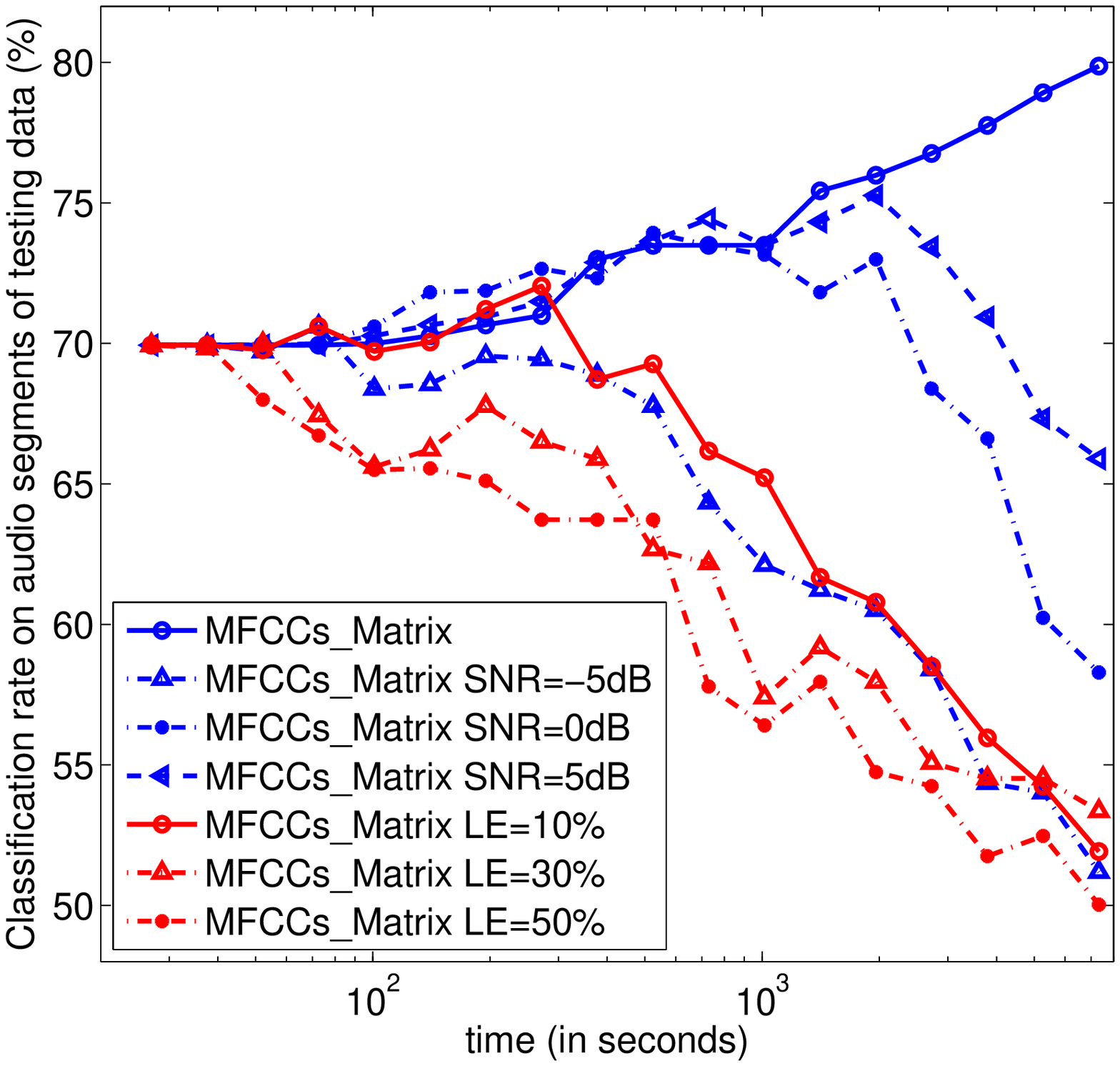}
}
\hspace{0.5in}
\subfigure[]{
\includegraphics[width=0.45\textwidth]{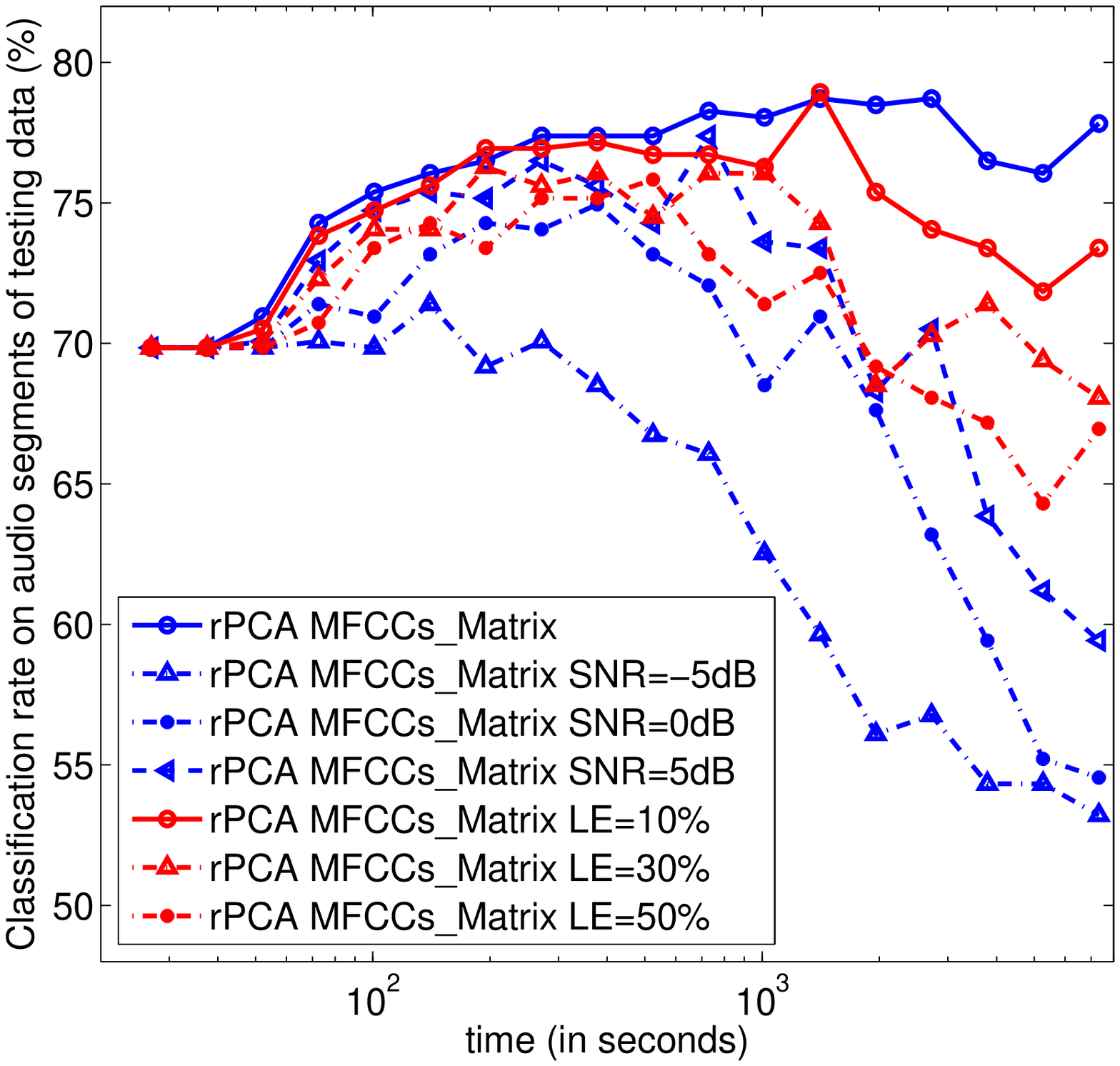}
}
\caption{(a) and (b): Comparisons of robust PCA extracted low-rank
features and MFCCs matrices in applause/non-applause segments classification.
(c) and (d): Comparisons of robust PCA extracted low-rank
features and MFCCs matrices in laugh/non-laugh segments classification.}
\label{fig:PerformancCompar}       
\end{figure*}

The classification accuracy of the one second audio segments is used to evaluate the performance of the methods. Fig.~\ref{fig:PerformancCompar}
shows the performances of the methods with different matrix features under different noise conditions as the functions of the training time used in Algorithm~\ref{algo:APG}. It can be seen that the original MFCCs matrix feature is not robust to noises, especially random large errors. If 10\% of the elements of the MFCCs matrix feature are corrupted with random large errors, then generally there would be a decrease of 25\% in
audio segments classification accuracy, while for robust PCA extracted low-rank features, the decrease
are 5\% in average. For WGN, the robust PCA features also perform better than original features, although not so sharp as in the situation of large errors. The experiments show that the low-rank components are more robust to noises and errors than the original features.

\begin{table*}
\centering
\caption{Performance comparison between our approach and SVM classification on long vector method for applause/non-applause segment classification.}
\label{tab:AppCompareWithSVM}       
\begin{tabular}{|c|c|c|c|c|c|c|c|}
\hline
Approach & Normal & SNR=-5dB & SNR=0dB & SNR=5dB & LE=10\% & LE=30\% & LE=50\%\\
\hline
SVM+LV  &  81.88\%  & 64.07\% & 64.07\% & 64.07\% & 64.07\% & 64.07\% & 64.07\% \\
\hline
APG+MFCCs\_Matrix  &82.76\%&  51.11\%&  55.87& 61.76\%&  52.78\%&  52.10\%&  51.16\% \\
\hline
SVM+rPCA LV & 81.88\%&  64.07\%&  64.07\%&  64.07\%& 81.77\%&  81.55\%&  81.43\% \\
\hline
APG+rPCA MFCCs\_Matrix & 82.17\%&  54.44\%&  61.75\%&  70.47\%&  80.33\%&  76.22\%&  72.96\% \\
 \hline
\end{tabular}
\end{table*}

\begin{table*}
\centering
\caption{Performance comparison between our approach and SVM classification on long vector method for laugh/non-laugh segment classification.}
\label{tab:LauCompareWithSVM}       
\begin{tabular}{|c|c|c|c|c|c|c|c|}
\hline
Approach & Normal & SNR=-5dB & SNR=0dB & SNR=5dB & LE=10\% & LE=30\% & LE=50\%\\
\hline
SVM+LV  &  81.88\%  & 60.01\%&  60.01\%&  60.01\%& 60.01\%&  60.01\%&  60.01\% \\
\hline
APG+MFCCs\_Matrix  &90.02\%&  53.03\%&  63.64\%& 70.07\%&  54.30\%&  52.47\%&  52.59\% \\
\hline
SVM+rPCA LV & 75.06\%&  60.01\%&  60.01\%&  60.01\%& 74.81\%&  74.97\%&  74.56\% \\
\hline
APG+rPCA MFCCs\_Matrix & 85.84\%&  54.36\%&  67.71\%&  76.97\%&  84.76\%&  80.24\%&  77.50\% \\
 \hline
\end{tabular}
\end{table*}

We also compare our method with the state-of-the-art SVM classifier with long vector feature (650 dimension) obtained by vectorizing the matrix. The results are summarized in Table~\ref{tab:AppCompareWithSVM} and Table~\ref{tab:LauCompareWithSVM} for applause/non-applause and laugh/non-laugh classification respectively. The results show that the SVM become useless under 5dB wight noise and 10\% large corruptions, while our methods still works. But for the low-rank component, the SVM performs better on some situations for which is due to the robustness of the features.


\section{Conclusions}
\label{sec:Conclusions}
In this work, we present a novel framework based on trace norm minimization for audio segment classification. The novel method unified feature extraction and pattern classification into the same framework. In this framework, robust PCA extracted low-rank component of original signal is more robust to corrupted noise and errors, especially to random large errors. We also introduced online learning algorithms for matrices classification tasks. We obtain the closed-form updating rules of the weight matrix and the bias. We derive the explicit form of the Lipschitz constant, which saves the computation burden in searching step size. Experiments show that even the percent of the original feature elements corrupted with random large errors is up to 50\%, the performance of the robust PCA extracted features almost have no decrease. In future work, we plan to test this robust feature in other audio or speech processing related applications and extend robust PCA, even trace norm minimization related methods from matrices to the more general multi-way arrays (tensors). Some work related to learning methods are also worth considering, such that the alternating between minimization with respect to weight matrix and bias may results in fluctuation of target value (even in batch mode), thus optimization algorithm that minimization jointly on weight matrix and bias are required; for multi-classification problems with more classes, some hierarchy methods may be introduced to improve the classification accuracy.


%
%
%

\end{document}